%% file: main.tex
\documentclass[twoside,11pt]{article}
\usepackage{jmlr2e}
\usepackage[utf8]{inputenc}
\usepackage[T1]{fontenc}
\usepackage{url}
\usepackage{booktabs}
\usepackage{amsfonts}
\usepackage{nicefrac}
\usepackage[expansion=false]{microtype}
\usepackage{xcolor}
\usepackage{xspace}
\input{preamble}
\input{macros}
\usepackage{cleveref}

\begin{document}
\firstpageno{1}
\title{Conformal Cascade: Distribution-Free Accuracy Guarantees for Multi-Tier LLM Inference}

\author{%
\name Yifan Dou \email yd25b@fsu.edu \\
\addr Department of Computer Science, Florida State University
\AND
\name Shikan Lian \email shikanlian@gmail.com 
\AND
\name Shibo Li \email shiboli@cs.fsu.edu \\
\addr 
Department of Computer Science, Florida State University
}

\begingroup
\renewcommand\thefootnote{\fnsymbol{footnote}}
\maketitle
\endgroup

\input{sections/abstract}
\input{sections/introduction}
\input{sections/background}
\input{sections/method}
\input{sections/related_work}
\input{sections/experiments}

\input{sections/discussion}

\input{sections/conclusion}

\newpage
\bibliography{references}
\newpage
\appendix
\input{sections/appendix}
\end{document}

%% file: preamble.tex
\usepackage{amsmath}
\usepackage{amssymb}
\usepackage{bm}

\usepackage{algorithm}
\usepackage{algpseudocode}

\usepackage{graphicx}

\usepackage{multirow}
\usepackage{makecell}
\usepackage{pifont}
\usepackage[table]{xcolor}

\graphicspath{{figures/}}

%% file: macros.tex
\newtheorem{assumption}[theorem]{Assumption}
\newcommand{\Cal}{\mathcal{C}}
\newcommand{\Dcal}{\mathcal{D}_{\mathrm{cal}}}
\newcommand{\Acal}{\mathcal{A}}
\newcommand{\Xcal}{\mathcal{X}}
\newcommand{\Ycal}{\mathcal{Y}}
\newcommand{\Prob}{\mathrm{Pr}}
\newcommand{\expect}{\mathbb{E}}
\newcommand{\ytrue}{y_{\mathrm{true}}}
\newcommand{\indic}{\mathbb{1}}

%% file: sections/abstract.tex
\begin{abstract}
Large language model (LLM) cascades reduce inference cost by routing easy queries to a small model and deferring hard queries to a larger one. Production cascades govern this deferral through a confidence threshold, but LLM confidence scores are miscalibrated, the threshold must be tuned per model pair and per domain, and no setting yields a formal bound on cascade accuracy. We introduce \textbf{Conformal Cascade} (CC), a multi-tier inference framework that uses conformal prediction set size as the deferral rule: accept when the calibrated set collapses to a single answer, defer otherwise. The procedure delivers a distribution-free, finite-sample accuracy guarantee. By a per-tier union bound, the prediction set at the accepting tier covers the correct answer with probability at least $1 - K\alpha$ for any user-specified $\alpha$; under a selection-preservation condition (consistent with, but not strictly implied by, our marginal coverage results), the bound tightens to $1 - \alpha$. We further characterise expected cascade cost as an explicit function of $\alpha$ and the calibration-set acceptance rate. Across 18 multiple-choice benchmarks spanning science, medicine, commonsense, and standardized exams, evaluated on two-tier cascades drawn from four open-weight model families, CC strictly improves over the strongest calibration-tuned heuristic cascade on the majority of family--benchmark pairs, with the largest gains on reasoning-heavy benchmarks where majority vote is unreliable; on easier benchmarks the cascade commits the vast majority of queries to the small model at no accuracy cost. Extension to open-ended generation requires an answer-clustering step that we leave for future work. The method requires no model training and only black-box API access.
\end{abstract}

%% file: sections/introduction.tex
\section{Introduction}
\label{sec:intro}

Deploying large language models at scale forces a sharp trade-off between answer quality and inference cost. Serving a frontier model on every query is often prohibitive: at current token prices, inference dominates the operating expense of any high-traffic product, and each expensive call consumes accelerator time that could serve simpler queries cheaply elsewhere. LLM cascades offer the dominant production response: a query runs first on a small model, and escalates to a larger one only when needed \citep{chen2023frugalgpt,ong2025routellm,madaan2024automix}. What determines whether a cascade works is the \emph{deferral rule}: at each tier, the decision to return the current answer or hand the query to the next model. Production cascades implement this decision as a threshold on a confidence signal such as softmax probability \citep{ong2025routellm}, self-consistency agreement \citep{wang2023selfconsistency,valkanas2025c3po}, or a learned router score \citep{shnitzer2023routing}: accept if the signal exceeds a cutoff $\theta$, defer otherwise.

Confidence-threshold deferral has a well-documented calibration problem. LLMs are overconfident in ways that depend on domain and input format, sharpest on the reasoning-heavy queries that benefit most from escalation \citep{kadavath2022languagemodels,xiong2024uncertainty}, and no choice of $\theta$ separates confidently-wrong inputs from genuinely easy ones. The deeper issue is that the cascade's error rate becomes an uncontrolled random variable: realized accuracy depends on the unknown joint distribution of queries, confidences, and correctness, and no analytical relationship maps a target error rate onto a threshold. Practitioners cannot answer the question that matters for deployment, ``if I promise my users at most $5\%$ wrong answers, what should $\theta$ be?'', and must re-sweep $\theta$ on a labelled split for every new model pair, benchmark, or difficulty distribution.

We propose \textbf{Conformal Cascade} (CC), which replaces the confidence threshold with a \emph{conformal prediction set}. At each tier, split conformal calibration \citep{vovk2005algorithmic,papadopoulos2002inductive} turns a held-out labelled set into a threshold that satisfies a distribution-free coverage guarantee \emph{at that tier}: for any new query, the per-tier prediction set built from this threshold contains the correct answer with probability at least $1 - \alpha$, for a user-specified $\alpha$. The \emph{size} of this set, not the value of a confidence score, drives deferral. A set that collapses to a single answer declares the model conformally certain on this input, and the cascade accepts; a set with multiple candidates declares the model conformally uncertain, and the cascade escalates. Deferral cannot itself cause an error: only the acceptance event $|\mathcal{C}| = 1$ commits a wrong answer, so the conformal guarantee binds at whichever tier the cascade commits, and per-tier guarantees compose into a cascade-level coverage bound. The calibrated threshold is fully determined by $\alpha$, which reads directly as the user's error budget, and the expected cascade cost is a closed-form function of $\alpha$ and the calibration-set acceptance rate.

We make three contributions. We formulate the conformal cascade for arbitrary $K$ tiers, where prediction-set size replaces the confidence score, and prove a cascade-level coverage bound (Theorem~\ref{thm:cascade-coverage}): the accepting tier's prediction set covers the correct answer with probability at least $1 - K\alpha$ in the worst case, tightening to $1 - \alpha$ under a selection-preservation condition consistent with our marginal-coverage measurements. We then derive a closed-form expression for expected cascade cost as a function of $\alpha$ (Theorem~\ref{thm:cost-coverage}). Empirically, we evaluate two-tier instantiations ($K=2$) across four open-weight model families (Llama, Gemma, Ministral, Phi) and 18 multiple-choice benchmarks; CC strictly improves over the strongest calibration-tuned heuristic cascade on the majority of family--benchmark pairs, with the largest gains on reasoning-heavy tasks where majority vote is unreliable: examples include Ministral on MathQA, Phi on GPQA-Diamond and MedMCQA, and Gemma on TruthfulQA. The framework requires a finite, extractable answer set; extension to open-ended generation through answer clustering is left for future work. The method imposes no architectural constraints beyond self-consistency-style sampling, requires neither training nor logprob access, and exposes one interpretable hyperparameter.

%% file: sections/background.tex
\section{Background and Problem Setup}
\label{sec:background}

\paragraph{Split conformal prediction.} Given a nonconformity score $s : \Xcal \times \Ycal \to \mathbb{R}$ (lower indicates better fit), a calibration set $\Dcal = \{(x_i, y_i)\}_{i=1}^{n}$ drawn exchangeably with the test point, and a target miscoverage level $\alpha \in (0,1)$, split conformal prediction \citep{papadopoulos2002inductive,vovk2005algorithmic,lei2018distribution,angelopoulos2023gentle} sets the threshold
\begin{equation}
    \hat{q} = \mathrm{Quantile}\!\left(\{s_i\}_{i=1}^{n};\, \tfrac{\lceil (n+1)(1-\alpha) \rceil}{n}\right)
    \label{eq:quantile}
\end{equation}
and returns the prediction set $\Cal(x) = \{y \in \Ycal : s(x, y) \le \hat{q}\}$, which satisfies the distribution-free, finite-sample guarantee:
\begin{theorem}[Marginal coverage; \citealp{vovk2005algorithmic}]
\label{thm:marginal-coverage}
$\Prob[y_{n+1} \in \Cal(x_{n+1})] \ge 1 - \alpha$ whenever the calibration and test points are exchangeable.
\end{theorem}
The upper bound $1 - \alpha + 1/(n+1)$ holds when scores are almost surely distinct \citep{lei2018distribution}. Conditional coverage (per-$x$ at level $1-\alpha$) is impossible without further assumptions \citep{vovk2012conditional,lei2014distribution,barber2021limits}.

\paragraph{LLM cascades.} A $K$-tier LLM cascade is a sequence of models $m_1, \ldots, m_K$ ordered by increasing capability and per-call cost $c_1 < \cdots < c_K$. Tier $k$ produces one or more responses; a deferral rule then decides between accepting tier $k$'s answer and escalating to tier $k+1$. If the cascade accepts at tier $k^*$, the incurred cost is $\mathrm{Cost}(x) = \sum_{k=1}^{k^*} c_k$, since a cheap-model response is always produced before an expensive one is requested.

\paragraph{Problem.} For queries with a finite answer set $\Acal(x)$ (multiple choice, or extractable short answers), let $\ytrue \in \Acal(x)$ be the true answer and $\Cal_{k^*}(x)$ the conformal prediction set at the accepting tier. The goal is to minimise expected inference cost $\expect[\mathrm{Cost}(X)]$ subject to
\begin{equation}
    \Prob\!\left[\ytrue \in \Cal_{k^*}(X)\right] \;\ge\; 1 - \alpha,
    \label{eq:goal}
\end{equation}
where the coverage target is on the prediction set at whichever tier the cascade commits, not at a fixed tier index; this is what makes the guarantee a cascade-level statement.

%% file: sections/method.tex
\section{Method}
\label{sec:method}

We describe each component of the Conformal Cascade. Sections~\ref{sec:score}--\ref{sec:calib} introduce the black-box nonconformity score and its per-tier calibration. Section~\ref{sec:alg} states the cascade algorithm. Sections~\ref{sec:theory}--\ref{sec:cost-coverage} establish the cascade-level coverage guarantee and the cost--accuracy trade-off. Proofs, auxiliary lemmas, and the computational-complexity breakdown are collected in Appendix~\ref{app:method-details}.

The core idea is to replace the signal-plus-threshold deferral rule of Section~\ref{sec:intro} with a calibrated prediction set at each tier, whose size is the signal. A confidence threshold $\theta$ has no formal tie to an error rate; the conformal threshold $\hat q_k$ has the tie supplied by Theorem~\ref{thm:marginal-coverage}. The user-specified parameter $\alpha$ then appears in two roles: as a coverage level (Theorem~\ref{thm:cascade-coverage}) and as a handle on expected cost (Theorem~\ref{thm:cost-coverage}).

\subsection{Frequency-Based Nonconformity Score}
\label{sec:score}

Split conformal prediction (Theorem~\ref{thm:marginal-coverage}) requires a real-valued nonconformity score on $(x, a)$ but is silent about where the score comes from. Under black-box API access, where token-level probabilities are typically not exposed, the natural choice is the sampling frequency of $a$ under repeated decoding from $m_k$. A logprob-based variant $s_k(x, a) = 1 - p_k(a \mid x)$ slots into the same framework with finer resolution; we develop the frequency version because it preserves the black-box interface assumed by many production deployments and discuss the comparison to the logprob alternative in Appendix~\ref{app:score-details}.

For tier $k$, we draw $N$ independent responses from $m_k$ at temperature $\tau > 0$ (the self-consistency regime of \citet{wang2023selfconsistency}), giving extracted answers $\hat{y}_1^{(k)}, \ldots, \hat{y}_N^{(k)}$. When the answer set $\Acal(x)$ varies per query (TruthfulQA has $|\Acal(x)| \in \{4, \ldots, 10\}$), all quantities below are defined per-query on $\Acal(x)$. The frequency-based nonconformity score for $a \in \Acal(x)$ is
\begin{equation}
    s_k(x, a) \;=\; 1 - \frac{\bigl|\{i : \hat{y}_i^{(k)} = a\}\bigr|}{N}.
    \label{eq:score}
\end{equation}
An answer sampled often scores low (high conformity); a rarely sampled answer scores high (low conformity). The score uses only black-box API access, reuses the $N$ samples already produced under self-consistency decoding, and composes naturally with a majority-vote tiebreaker when multiple answers survive calibration; its validity as a nonconformity measure follows \citet{kumar2023conformal}. Appendix~\ref{app:score-details} (Lemma~\ref{lem:exchangeability}) certifies that exchangeability of queries lifts to exchangeability of scores under per-query independent Monte-Carlo draws.

\subsection{Per-Tier Conformal Calibration}
\label{sec:calib}

Calibration converts the score $s_k$ into a prediction-set construction rule with a concrete coverage guarantee at level $\alpha$. For each tier $k$, we calibrate a threshold $\hat{q}_k$ on a held-out calibration set $\Dcal = \{(x_i, y_i)\}_{i=1}^{n}$ by (i) drawing $N$ samples from $m_k$ for each calibration query and computing $s_i^{(k)} = s_k(x_i, y_i)$, (ii) sorting the scores $s_{(1)}^{(k)} \le s_{(2)}^{(k)} \le \cdots \le s_{(n)}^{(k)}$, and (iii) setting $\hat{q}_k = s_{(\lceil (1-\alpha)(n+1) \rceil)}^{(k)}$. The prediction set at tier $k$ for a new query $x$ is
\begin{equation}
    \Cal_k(x) \;=\; \{a \in \Acal : s_k(x, a) \le \hat{q}_k\}.
    \label{eq:tier-set}
\end{equation}
Operationally, $\hat q_k$ is the smallest nonconformity score the calibration set permits at coverage level $\alpha$: answers scoring below it pass the conformal filter, those above are filtered out. By Theorem~\ref{thm:marginal-coverage}, under exchangeability $\Prob[\ytrue \in \Cal_k(x)] \ge 1 - \alpha$ at each tier in isolation. Calibration-size choice, quantile tie-breaking, and computational complexity are detailed in Appendix~\ref{app:calib-details}.

\subsection{The Conformal Cascade Algorithm}
\label{sec:alg}

With per-tier prediction sets in hand, the cascade asks one question at each tier: is the calibrated set small enough to commit on? Algorithm~\ref{alg:cc} specifies the procedure. An acceptance threshold $\kappa \ge 1$ governs how small a set counts as a commitment (default $\kappa = 1$).

\begin{algorithm}[t]
\caption{Conformal Cascade (CC)}
\label{alg:cc}
\begin{algorithmic}[1]
\Require Query $x$; models $m_1, \ldots, m_K$; calibrated thresholds $\hat{q}_1, \ldots, \hat{q}_K$; acceptance parameter $\kappa$
\Ensure Answer $\hat{y}$, accepting tier $k^*$
\For{$k = 1, \ldots, K$}
    \State Draw $N$ samples from $m_k$ on $x$
    \State Compute $s_k(x, a)$ for each $a \in \Acal$ \Comment{Eq.~\eqref{eq:score}}
    \State Construct $\Cal_k(x) = \{a \in \Acal : s_k(x, a) \le \hat{q}_k\}$ \Comment{Eq.~\eqref{eq:tier-set}}
    \If{$1 \le |\Cal_k(x)| \le \kappa$}
        \State $\hat{y} \gets \arg\min_{a \in \Cal_k(x)} s_k(x, a)$ \Comment{Most frequent answer in set}
        \State \Return $\hat{y}, k^* = k$
    \EndIf
\EndFor
\State $\hat{y} \gets \arg\min_{a \in \Acal} s_K(x, a)$ \Comment{Fallback at last tier}
\State \Return $\hat{y}, k^* = K$
\end{algorithmic}
\end{algorithm}

\paragraph{Interpretation.} The set size measures tier-level uncertainty through the conformal filter. A set of size one (the cascade accepts, the model is \emph{conformally certain}) commits to an answer; size larger than $\kappa$ (\emph{conformally uncertain}) escalates. An empty set is an out-of-distribution signal and is also treated as a defer. Relaxed acceptance ($\kappa > 1$) accepts when the set is merely small, breaks argmin ties lexicographically on $\Acal$, and trades a controlled amount of correctness for earlier acceptance (Section~\ref{sec:cost-coverage}). Our implementation runs the fallback (line~11) over $\Cal_K(x)$ when $|\Cal_K(x)| \ge 1$ and over $\Acal$ only when the tier-$K$ set is empty; this affects boundary behaviour on benchmarks where the cascade defers all queries but does not change the offline coverage guarantee (Appendix~\ref{app:weak-families}).

\paragraph{Worked example.} Consider a 4-choice query $x$ on which tier-1 self-consistency draws ($N = 16$) yield extracted-answer counts $\{A: 10, B: 4, C: 1, D: 1\}$ and tier-1 score values $s_1(x, A) = 6/16, s_1(x, B) = 12/16, s_1(x, C) = 15/16, s_1(x, D) = 15/16$. With a tight calibrated threshold $\hat q_1 = 0.50$ (a small $\alpha$), only $A$ passes: $\Cal_1(x) = \{A\}$, the cascade accepts at tier 1, and the per-query cost is $c_1$. With a looser $\hat q_1 = 0.80$, the set becomes $\Cal_1(x) = \{A, B\}$: the cascade defers under $\kappa = 1$ and pays $c_1 + c_2$, but with $\kappa = 2$ it accepts and returns $\arg\min_{a \in \{A, B\}} s_1(x, a) = A$. The same query commits at tier 1 or escalates depending on $\alpha$ and $\kappa$; the conformal threshold turns the raw frequency distribution into a calibrated commit/defer signal.

\subsection{Theoretical Analysis}
\label{sec:theory}

Let $K^*(x) \in \{1, \ldots, K\}$ denote the first tier at which the cascade commits in Algorithm~\ref{alg:cc} (defined as $K$ when the fallback fires), and let $A = \{|\Cal_{K^*(x)}(x)| = 1\}$ be the \emph{acceptance event}.

\begin{theorem}[Cascade coverage]
\label{thm:cascade-coverage}
Let $(x, \ytrue)$ be exchangeable with each tier-$k$ calibration set, for $k = 1, \ldots, K$. In a $K$-tier Conformal Cascade with per-tier level $\alpha$ and acceptance parameter $\kappa = 1$,
\begin{equation}
    \Prob\!\left[\,\hat y \ne \ytrue,\, A\,\right] \;\le\; K \cdot \alpha,
    \label{eq:cascade-coverage}
\end{equation}
so per-tier level $\alpha / K$ ensures cascade error at most $\alpha$.
\end{theorem}

The proof partitions $A$ into the $K$ disjoint first-commit events $A_k = \{K^*(x) = k\}$ and applies Theorem~\ref{thm:marginal-coverage} on each (Appendix~\ref{app:thm-cascade}). Two facts about the bound matter for practice. Deferral cannot itself cause an error: only the acceptance event $A$ commits a wrong answer, escalation only incurs additional cost, and the cascade error decomposes over the disjoint commit events $\{A_k\}_{k=1}^K$. Per-tier guarantees therefore compose into a cascade-level bound by a union bound on these disjoint events. The factor $K$ is the union-bound worst case; tighter statements hold under selective-inference conditions \citep{bates2023testing,angelopoulos2023gentle}, captured by the next assumption.

\begin{assumption}[Selection preservation]
\label{asm:selection-preservation}
For each tier $k$, $\Prob[\,\ytrue \in \Cal_k(x) \mid K^*(x) = k\,] \ge 1 - \alpha$.
\end{assumption}

Under Assumption~\ref{asm:selection-preservation} the cascade error is bounded by $\alpha$ directly. The first-commit event $\{K^*(x) = k\}$ is label-free---a function of $x$ and the calibration scores, not $\ytrue$---which places it in the family studied by selective conformal inference; the assumption asks that this label-free selection preserve per-tier coverage, and is not derivable from split CP alone. Section~\ref{sec:coverage} reports marginal miscoverage tracking $\alpha$ rather than $K\alpha$, which rules out the $K\alpha$ bound being tight here but is only \emph{consistent with}, not strictly implied by, Assumption~\ref{asm:selection-preservation} (which would require per-tier conditional miscoverage). Two further notes on scope: Theorem~\ref{thm:cascade-coverage} bounds the joint event \{error, acceptance\}; on the fallback path the cascade returns an argmin without an active conformal commit and inherits no per-tier guarantee (rare in practice---Appendix~\ref{app:weak-families}). For $\kappa > 1$ the union bound controls \emph{set} coverage, and the top-1 error picks up an additional slack from the argmin tiebreaker (Remark~\ref{rem:relaxed-kappa}, Appendix~\ref{app:thm-cascade}). The relaxed-acceptance variant and the heterogeneous-$\alpha_k$ extension are also in Appendix~\ref{app:thm-cascade}.

\subsection{Cost--Coverage Trade-off}
\label{sec:cost-coverage}

The accuracy guarantee comes paired with a closed-form expression for expected cost.

\begin{theorem}[Cost--coverage, two-tier]
\label{thm:cost-coverage}
For a two-tier cascade $(K = 2)$ with costs $c_1 < c_2$ and $\kappa = 1$,
\begin{equation}
    \expect[\mathrm{Cost}(X)] \;=\; c_1 + c_2 \cdot \Prob\!\left[\,|\Cal_1(X)| \ne 1\,\right].
    \label{eq:cost-closed-form}
\end{equation}
\end{theorem}

Expected cost depends only on the calibration-set non-acceptance rate, which is observable on labelled calibration data without further test queries. The asymptotes are clean: as $\alpha \to 0$ prediction sets grow and cost approaches $c_1 + c_2$; as $\alpha \to 1$ sets shrink to singletons and cost approaches $c_1$. Cost is monotone in $\alpha$ on the admissible range, and the cascade strictly improves over the always-defer baseline ($\expect[\mathrm{Cost}] < c_1 + c_2$) whenever the cheap model produces any singleton sets on calibration (Proposition~\ref{prop:save-cost}). Improving over Always-Strong (which skips tier 1 and pays $c_2$) requires the stricter condition $\Prob[|\Cal_1(X)| = 1] > c_1 / c_2$, since the cascade pays $c_1$ per query for the tier-1 samples regardless of whether they trigger acceptance.

For $K > 2$ tiers, $\expect[\mathrm{Cost}(X)] = \sum_{k=1}^{K} (\sum_{j=1}^{k} c_j) \cdot p_k$ with $p_k = \Prob[K^*(X) = k]$ (Corollary~\ref{cor:k-tier-cost}, appendix). Combined with Theorem~\ref{thm:cascade-coverage}, this gives a practitioner recipe: for a target cascade error rate $\beta$, set $\alpha = \beta / K$ in the worst case or $\alpha = \beta$ under Assumption~\ref{asm:selection-preservation} as a deployment heuristic (Remark~\ref{prop:optimal-alpha}, appendix; finite-$N$ monotonicity is checked by a small calibration sweep, Appendix~\ref{app:cost-sensitivity}). The user fixes the error budget; the cascade picks the threshold.

Concretely, with $K = 2$ and $c_1 = 1$, $c_2 = 2.7$, an $80\%$ calibration acceptance rate gives $\expect[\mathrm{Cost}] = 1.54$ ($43\%$ saving over Always-Strong); the saving as a function of cost ratio $r = c_2/c_1$ is in Appendix~\ref{app:cost-sensitivity}.

%% file: sections/related_work.tex
\section{Related Work}
\label{sec:related}

Split conformal prediction \citep{vovk2005algorithmic,papadopoulos2002inductive,angelopoulos2023gentle,shafer2008tutorial} produces prediction sets with distribution-free finite-sample coverage and has been extended to regression \citep{romano2019conformalized,lei2018distribution}, classification with adaptive coverage \citep{romano2020classification,angelopoulos2021uncertainty}, jackknife-style resampling \citep{barber2021predictive}, and covariate shift \citep{tibshirani2019conformal,gibbs2021adaptive,angelopoulos2024pid}. Recent LLM-specific variants address the settings where standard scores fail: sampling-frequency nonconformity when logprobs are unavailable \citep{kumar2023conformal}, conformal stopping rules for open-ended generation \citep{quach2024conformal,deutschmann2024conformal}, conformal factuality for free-form generation \citep{mohri2024conformal,cherian2024conformal}, and conformal binary routing between a base LLM and a reasoning model \citep{su2025cprouter}. These methods cover single-model or two-model settings; none constructs conformal sets at every tier of a cascade, uses set \emph{size} as the deferral criterion, or proves coverage composition across tiers. On the cascade side, heuristic-threshold methods \citep{chen2023frugalgpt,madaan2024automix}, preference-learned routers \citep{ong2025routellm,ding2024hybrid}, benchmark-trained routers \citep{shnitzer2023routing}, mixture-of-thought cascades \citep{yue2024cascade}, and theoretical analyses of confidence-based cascade deferral \citep{jitkrittum2023cascade,gupta2022cascades} form the current production and analytical landscape; none of these provides a distribution-free finite-sample accuracy guarantee that holds at every tier, and each requires labelled data to tune per deployment.

\paragraph{Comparison to C3PO.} The closest prior work is C3PO \citep{valkanas2025c3po}, also a conformal-prediction approach to LLM cascades but on an orthogonal axis: C3PO bounds the probability that per-query inference cost exceeds a budget ($\Pr[\mathrm{cost} > C^*] \le \alpha$) using a confidence-threshold deferral rule and self-supervised calibration, whereas CC bounds the probability that the accepting tier's set misses the correct answer using a binary set-size test and labelled calibration. C3PO suits a fixed compute budget seeking cost control; CC suits a fixed error budget seeking accuracy control. Appendix~\ref{app:c3po-table} gives the row-by-row design contrast.

Underneath these threshold-based cascades lies a calibration problem: neural-network confidence is systematically miscalibrated \citep{guo2017calibration}, and elicited LLM confidence in particular drifts on reasoning tasks \citep{kadavath2022languagemodels,xiong2024uncertainty,lin2022teaching,tian2023just,kuhn2023semantic,lin2024generating}. Classical post-hoc rescaling (Platt scaling, temperature scaling) adjusts score magnitudes without restoring a distribution-free guarantee, whereas conformal prediction sidesteps calibration by bounding set-level coverage directly. Beyond calibration, adaptive compute allocation at inference time is an active area: \citet{snell2025scaling} show that the optimal allocation depends on query difficulty, \citet{setlur2025scaling} argue that verification-based strategies outperform verification-free sampling, and \citet{lightman2024verify} show that verifiers trained at the step level can substantially improve reasoning reliability; CC is verification-based in this sense, with the verifier being a conformal set-size check whose correctness carries a distribution-free guarantee. Our frequency-based score builds on self-consistency decoding \citep{wang2023selfconsistency,chen2024universal}, which is itself part of the chain-of-thought family \citep{wei2022chainofthought}.

%% file: sections/experiments.tex
\section{Experiments}
\label{sec:exp}

\subsection{Setup}
\label{sec:setup}

\paragraph{Models.} We evaluate two-tier ($K=2$) cascades drawn from four open-weight instruction-tuned families: \textbf{Llama} \citep{dubey2024llama3}, \textbf{Gemma} \citep{gemma2024team}, \textbf{Ministral} \citep{mistralai2024ministral}, and \textbf{Phi} \citep{abdin2024phi3}; while the framework supports arbitrary $K$, the empirical sweep is restricted to $K=2$. Tier 1 is the smallest public instruction-tuned checkpoint; tier 2 is a substantially larger checkpoint from the same family (identifiers in Appendix~\ref{app:models}). For each query we draw $N = 16$ samples at temperature $\tau = 0.7$ from each tier via vLLM \citep{kwon2023vllm}. Per-call costs are fixed at $c_1 = 1$, $c_2 = 2.7$ across all four families; sensitivity to alternative $c_2/c_1$ is in Appendix~\ref{app:cost-sensitivity}.

\paragraph{Benchmarks.} We evaluate on 18 multiple-choice benchmarks covering science \citep{welbl2017sciq,khot2020qasc,clark2018arc,rein2024gpqa}, medicine \citep{pal2022medmcqa,jin2021medqa,jin2019pubmedqa}, broad knowledge \citep{hendrycks2021mmlu,wang2024mmlupro}, commonsense \citep{zellers2019hellaswag,bisk2020piqa,sap2019socialiqa,huang2019cosmosqa,clark2019boolq}, factuality \citep{lin2022truthfulqa}, standardized exams \citep{zhong2024agieval}, mathematics \citep{amini2019mathqa}, and multilingual reading \citep{bandarkar2024belebele}. Answer sets range from four to ten options. Each benchmark splits $30\%/70\%$ calibration/test with seed $42$, giving 72 family--benchmark pairs; full statistics in Table~\ref{tab:benchmarks}.

\paragraph{Baselines.} CC is compared against five label-light deferral rules that share its sample budget: (1) \textbf{Always-Weak} and (2) \textbf{Always-Strong} (majority vote at the corresponding tier, no tuning); (3) \textbf{Agreement cascade}, deferring if tier-1 top-answer agreement is below $\theta \in \{0.5, 0.6, 0.7, 0.8, 0.9\}$; (4) \textbf{Entropy cascade}, $\tau \in \{-1.5, -1.0, -0.5, -0.3\}$; (5) \textbf{Random defer}, $p \in \{0.2, 0.5, 0.8\}$. The heuristic grid contains 12 tunable configurations. Learned routers \citep{ong2025routellm,shnitzer2023routing,ding2024hybrid} require labelled preference data and a separate parametric scorer outside CC's calibration regime; we discuss them in Section~\ref{sec:related} but do not run them as direct baselines.

\paragraph{CC and reporting protocol.} CC is evaluated over a 15-point grid: $\alpha \in \{0.05, 0.10, 0.15, 0.20, 0.30\}$ and $\kappa \in \{1, 2, 3\}$. For each family--benchmark pair, \textbf{Best Heuristic} is the argmax test-set accuracy over the 12-configuration heuristic grid; \textbf{CC Best} is the argmax over the 15-configuration CC grid. Both methods are thus reported at oracle-tuned operating points per benchmark, on grids of comparable size. Coverage (Table~\ref{tab:coverage}, Appendix~\ref{app:coverage}) and tier-1 acceptance (Table~\ref{tab:setsize}, Appendix~\ref{app:setsize}) are reported at the tuning-free anchor $\alpha = 0.10, \kappa = 1$, where $\alpha$ reads as the user's error budget.

\subsection{CC strictly improves over heuristic cascades on accuracy}
\label{sec:main-results}

Table~\ref{tab:main-family} reports the per-benchmark, per-family accuracy gap $\Delta = \text{CC Best} - \text{Best Heuristic}$ across the 72 evaluated pairs at oracle-tuned operating points (both methods test-set argmax over their hyperparameter grids; deployment-time numbers at fixed $\alpha{=}0.10$, $\kappa{=}1$ are in Tables~\ref{tab:coverage}, \ref{tab:setsize}). CC strictly improves on 49 pairs, ties within $0.05$ pp on 6, and underperforms on 17. The six largest gains are Ministral/MathQA ($+4.6$), Phi/GPQA-Diamond ($+4.3$), Phi/MedMCQA ($+3.6$), Gemma/TruthfulQA ($+3.3$), Gemma/MedMCQA ($+3.0$), and Gemma/MedQA-USMLE-4 ($+2.8$). Every family has at least one gain $\ge +1.5$ pp, and all four families win the majority of their benchmarks. Full absolute accuracies are in Appendix~\ref{app:full-results}.

\begin{table}[t]
\caption{Accuracy gap $\Delta$ (in pp) between CC Best and Best Heuristic, per benchmark and model family. Both methods are tuned by test-set argmax over their hyperparameter grids (15 configurations for CC, 12 for the heuristic family; see Section~\ref{sec:setup}). \textbf{Bold} highlights $\Delta \ge +1.0$ pp; underlined highlights $\Delta \le -1.0$ pp. Row and column totals summarize wins~(W), ties within $0.05$ pp (T), and losses~(L). Gains are broad (49 wins of 72 pairs) and concentrate on reasoning-heavy benchmarks (GPQA-Diamond, MathQA, MMLU-Pro, MMLU, TruthfulQA, MedMCQA, MedQA-USMLE-4).}
\label{tab:main-family}
\centering
\small
\setlength{\tabcolsep}{4pt}
\begin{tabular}{lrrrr}
\toprule
Benchmark & Llama & Gemma & Ministral & Phi \\
\midrule
GPQA-Diamond   & $\bm{+1.5}$         & $+0.7$               & $+0.0$              & $\bm{+4.3}$         \\
MathQA         & $\bm{+1.7}$         & $+0.4$               & $\bm{+4.6}$         & $\underline{-1.4}$  \\
MMLU-Pro       & $\bm{+1.4}$         & $\bm{+2.4}$          & $\bm{+2.6}$         & $\bm{+2.0}$         \\
MMLU           & $+0.4$              & $\bm{+1.6}$          & $+0.7$              & $\bm{+1.5}$         \\
TruthfulQA   & $\bm{+1.9}$         & $\bm{+3.3}$          & $+0.7$              & $\bm{+1.2}$         \\
MedMCQA        & $+0.3$              & $\bm{+3.0}$          & $+0.8$              & $\bm{+3.6}$         \\
MedQA-USMLE-4  & $+0.4$              & $\bm{+2.8}$          & $\bm{+1.0}$         & $\bm{+2.2}$         \\
AGI-Eval       & $+0.2$              & $\bm{+2.0}$          & $\bm{+1.1}$         & $\underline{-4.5}$  \\
HellaSwag      & $+0.2$              & $\bm{+2.7}$          & $+0.4$              & $+0.4$              \\
ARC-Challenge  & $-0.1$              & $-0.1$               & $+0.9$              & $-0.4$              \\
PIQA           & $+0.2$              & $+0.0$               & $-0.6$              & $\bm{+1.3}$         \\
Social-i-QA    & $+0.2$              & $+0.1$               & $+0.7$              & $\bm{+1.0}$         \\
Cosmos-QA      & $-0.1$              & $-0.2$               & $+0.2$              & $+0.3$              \\
BoolQ          & $+0.4$              & $+0.0$               & $+0.3$              & $+0.1$              \\
QASC           & $+0.0$              & $\bm{+1.6}$          & $+0.2$              & $-0.1$              \\
SciQ           & $-0.1$              & $+0.0$               & $-0.1$              & $+0.2$              \\
Belebele-Eng   & $+0.2$              & $-0.3$               & $-0.2$              & $-0.5$              \\
PubmedQA       & $-0.2$              & $+0.0$               & $-0.7$              & $\underline{-1.4}$  \\
\midrule
W / T / L      & 13 / 1 / 4          & 11 / 4 / 3           & 13 / 1 / 4          & 12 / 0 / 6          \\
\bottomrule
\end{tabular}
\end{table}

\paragraph{Where the method is weaker.} Phi is the outlier on the losing side: large wins on reasoning and medicine (Phi/GPQA-Diamond $+4.3$, Phi/MedMCQA $+3.6$), but $-4.5$ pp on AGI-Eval and $-1.4$ pp on both MathQA and PubmedQA. The failure mode is diagnosable from the tier-1 set-size distribution (Section~\ref{sec:when-help}): on these benchmarks the Phi mini checkpoint spreads sampling probability so broadly that every tier-1 conformal set retains all (or nearly all) answer choices, the cascade defers $100\%$ of queries to tier 2, and the cascade's fallback selection rule under-performs Always-Strong's unrestricted majority on the same tier-2 samples (e.g., $45.7\%$ vs $51.0\%$ on Phi/AGI-Eval; see Appendix~\ref{app:weak-families}). CC is most useful when the cheap model's sampling distribution is concentrated enough to commit at tier 1 on easy queries and bimodal-or-narrow on hard queries to leave a small conformal set; uniformly diffuse tier-1 distributions, like Phi on AGI-Eval, eliminate both modes.

\subsection{The coverage guarantee holds empirically}
\label{sec:coverage}

Table~\ref{tab:coverage} reports the empirical set-miscoverage rate $\Prob[\ytrue \notin \Cal_{K^*(x)}(x)]$ for CC-Offline at $\alpha \in \{0.05, 0.10, 0.20, 0.30\}$ on the Llama cascade. At $\alpha = 0.10$, six of eight benchmarks meet the strict target within the $\pm 0.02$ finite-sample band; MedMCQA at $0.130$ and ARC-Challenge at $0.124$ exceed strict $\alpha$ by margins that are statistically significant given their test-set sizes (5.4$\sigma$ and 2.3$\sigma$, respectively), but both remain comfortably inside the worst-case $K\alpha = 0.20$ bound of Theorem~\ref{thm:cascade-coverage}. The same pattern (small overshoot at strict $\alpha$, well within $K\alpha$) holds on five additional cells across the other three families (Appendix~\ref{app:coverage} marks them with $^*$); every reported cell satisfies the union bound. The empirical regime is therefore between $\alpha$ and $K\alpha$, closer to the former. This rules out the worst-case $K\alpha$ bound being tight in our setting; it is consistent with, but does not directly verify, Assumption~\ref{asm:selection-preservation}, which would require the per-tier conditional miscoverage breakdown that our aggregate logs do not retain (Section~\ref{sec:theory}). Results for the other three families are in Appendix~\ref{app:coverage}.

\begin{table}[t]
\caption{Empirical set miscoverage for CC-Offline on the Llama cascade. \textbf{Bold} cells satisfy the strict $\alpha$ target within finite-sample fluctuation (miscoverage $\le \alpha + 0.02$). All non-bold cells lie inside the union-bound $K\alpha = 2\alpha$ tolerance from Theorem~\ref{thm:cascade-coverage}.}
\label{tab:coverage}
\centering
\small
\begin{tabular}{lcccc}
\toprule
Benchmark & $\alpha = 0.05$ & $\alpha = 0.10$ & $\alpha = 0.20$ & $\alpha = 0.30$ \\
\midrule
AGI-Eval       & \textbf{0.000} & \textbf{0.000} & \textbf{0.169} & \textbf{0.298} \\
MMLU-Pro       & \textbf{0.000} & \textbf{0.000} & \textbf{0.194} & \textbf{0.295} \\
MMLU           & \textbf{0.000} & \textbf{0.087} & \textbf{0.210} & \textbf{0.298} \\
TruthfulQA     & \textbf{0.000} & \textbf{0.000}  & 0.278          & 0.357           \\
GPQA-Diamond   & \textbf{0.000} & \textbf{0.000}  & \textbf{0.000} & \textbf{0.223} \\
MedMCQA        & \textbf{0.000} & 0.130           & 0.254          & 0.351           \\
ARC-Challenge  & 0.078          & 0.124          &\textbf{0.166}  & \textbf{0.207} \\
HellaSwag      & \textbf{0.000} & \textbf{0.065}  & \textbf{0.199} & \textbf{0.308} \\
\bottomrule
\end{tabular}
\end{table}

\subsection{The cost--accuracy Pareto frontier}
\label{sec:pareto}

Accuracy at the best-$\alpha$ operating point is only part of the story. In production, the practitioner-facing question is how CC trades cost for accuracy across a range of operating points. Figure~\ref{fig:pareto-mmlu} plots the cost--accuracy scatter on MMLU for three of the four cascades (Llama, Gemma, Ministral), sweeping every hyperparameter of every method; the Phi panel is in Appendix~\ref{app:pareto-all}.

\begin{figure}[t]
\centering
\includegraphics[width=\linewidth]{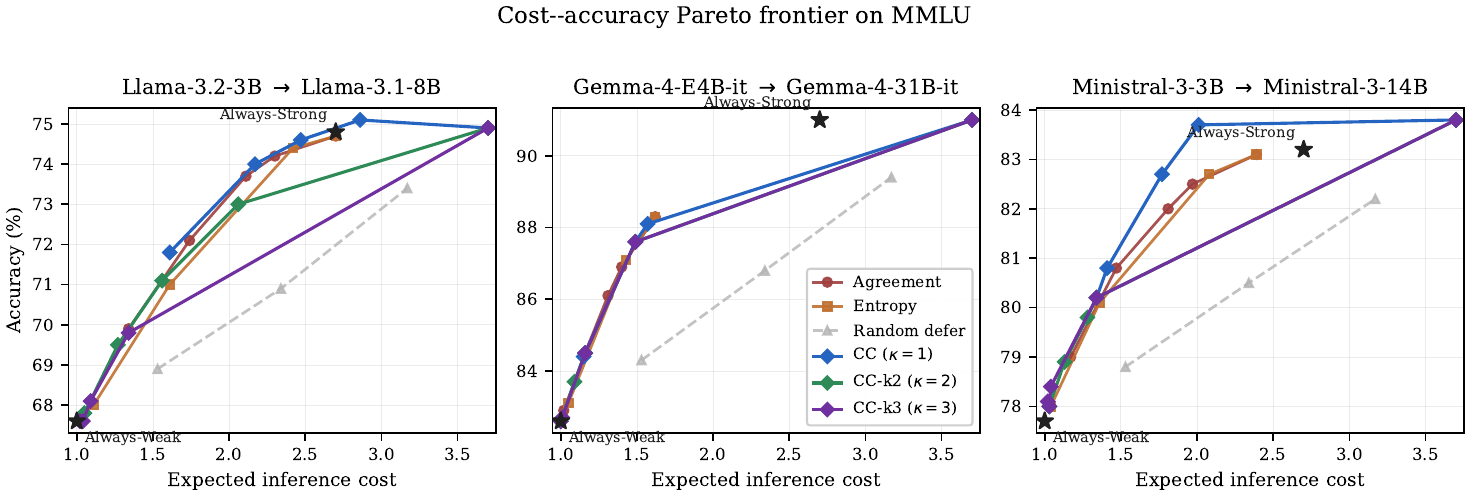}
\caption{Cost--accuracy Pareto frontier on MMLU for three cascades. Each method traces a curve through its hyperparameter sweep (points joined left-to-right by cost). CC ($\kappa = 1$, blue) dominates every heuristic on Llama and Ministral across the full cost range; on Gemma, where tier-1 is already at $82.6\%$, CC matches or slightly improves on Agreement. Relaxed-acceptance CC-k2 (green) and CC-k3 (purple) unlock cheaper operating points while retaining the conformal guarantee: on Llama, moving from $\kappa=1$ to $\kappa=3$ at $\alpha = 0.10$ cuts expected cost from $2.86$ to $1.34$ ($53\%$ reduction) at a $5.3$ pp accuracy cost.}
\label{fig:pareto-mmlu}
\end{figure}

On \textbf{Llama}, CC ($\kappa=1$) sits strictly above every heuristic for expected cost $\ge 1.5$ (e.g., at cost $2.86$, $75.1\%$ vs Always-Strong's $74.8\%$ at $2.70$); on \textbf{Ministral}, CC dominates Agreement across the mid-cost range (e.g., $83.7\%$ at cost $2.01$ vs Agree($\theta=0.8$)'s $82.5\%$ at $1.97$); on \textbf{Gemma}, where tier-1 already hits $82.6\%$, headroom is smaller and CC matches or slightly exceeds the heuristic frontier. Relaxing $\kappa$ to $\{2, 3\}$ accepts earlier and trades within-set selection accuracy for lower cost, moving along a second Pareto axis that no single agreement or entropy threshold tracks. Full per-family curves for the remaining 17 benchmarks are in Appendix~\ref{app:pareto-all}.

\subsection{Where the gains come from}
\label{sec:when-help}

Where in the data does CC's advantage concentrate? Two stratifications give complementary views. A difficulty breakdown on MMLU-Pro (Llama, $\alpha=0.10$, $100\%$ deferral to tier 2) shows CC's gap over Always-Strong sitting on the hard bucket: $7.2\%$ versus $3.4\%$ on queries where the strong model's per-query accuracy $\hat p < 0.25$, a $+3.8$ pp gain. On the easy and medium buckets the methods sit within $0.5$ pp of each other. The same hard-bucket-concentrated pattern holds on Ministral/MMLU and Ministral/MMLU-Pro (CC ahead by $4$--$6$ pp on hard queries); on Gemma/MMLU-Pro the strong model bottoms out at $0\%$ on hard queries and CC has no headroom to recover. Full per-bucket numbers for all four cells are in Appendix~\ref{app:difficulty}.

The second view is the tier-1 prediction-set distribution. Table~\ref{tab:setsize} reports the tier-1 acceptance rate ($|\Cal_1| = 1$) and median set size at $\alpha = 0.10$ on the Llama cascade. Acceptance ranges from $0\%$ on MMLU-Pro and GPQA-Diamond (all queries defer) to $95.6\%$ on Belebele-Eng (almost all queries commit at tier 1): the cascade automatically routes benchmarks where the small model is reliable to tier 1 and defers the rest, without per-benchmark tuning.

\begin{table}[t]
\caption{Tier-1 acceptance rate ($|\Cal_1| = 1$) at $\alpha = 0.10$ on the Llama cascade, ordered by acceptance. When the small model is broadly confident (top rows), the cascade commits most queries cheaply; when it is uniformly uncertain (bottom rows), the cascade defers everything. Full per-family distributions in Appendix~\ref{app:setsize}.}
\label{tab:setsize}
\centering
\small
\begin{tabular}{lccc}
\toprule
Benchmark & $|\Acal|$ & Tier-1 accept (\%) & Median $|\Cal_1|$ \\
\midrule
Belebele-Eng     & 4   & 95.6 & 1 \\
SciQ             & 4   & 94.0 & 1 \\
ARC-Challenge    & 4   & 85.6 & 1 \\
Cosmos-QA        & 4   & 85.6 & 1 \\
PIQA             & 2   & 80.3 & 1 \\
QASC             & 8   & 75.7 & 1 \\
BoolQ            & 2   & 72.8 & 1 \\
MMLU             & 4   & 31.3 & 2 \\
MedMCQA          & 4   & 22.0 & {2} \\
GPQA-Diamond     & 4   &  0.0 & 4 \\
MMLU-Pro         & 10  &  0.0 & 10 \\
\bottomrule
\end{tabular}
\end{table}

\paragraph{Diagnosing failures.} On Phi/AGI-Eval ($\Delta = -4.5$ pp), the tier-1 conformal sets never collapse to a singleton at any tested $\alpha$, the cascade defers $100\%$ of queries, and the gap to Always-Strong comes from the fallback selection rule (Appendix~\ref{app:weak-families}).

%% file: sections/discussion.tex
\section{Discussion}
\label{sec:discussion}

\paragraph{Deployment recipe.}
Setup is three steps: hold out a labelled calibration split per tier ($n \ge 200$ recommended), draw $N$ self-consistency samples per calibration query at the production temperature and store $\hat q_k$, and at inference run Algorithm~\ref{alg:cc}. The practitioner picks $\beta$ and sets $\alpha = \beta / K$ in the worst case or $\alpha = \beta$ under Assumption~\ref{asm:selection-preservation}. A label-free guard rail handles the diffuse-tier-1 regime (Phi/AGI-Eval, Section~\ref{sec:when-help}): if calibration tier-1 acceptance falls below the break-even rate $1/r$ ($r = c_2/c_1$; $r=2.7$ here gives $1/r \approx 0.37$, larger production ratios push it lower; Appendix~\ref{app:cost-sensitivity}), replacing CC with Always-Strong saves $c_1$ per query at no accuracy cost.

\paragraph{Scope, limitations, and impact.}
The guarantee is marginal, not conditional on query difficulty; strengthening this is a known hard problem. Calibration--test exchangeability comes from within-benchmark splitting, so covariate-shift deployment needs online recalibration. The framework requires a finite extractable answer set; open-ended generation needs an answer-clustering step. Aggregate-only logs preclude bootstrap intervals; small differences ($|\Delta| < 0.5$ pp on GPQA-Diamond at $n_{\text{test}}=139$) are point estimates. The auditable error budget reduces per-query compute when the small model is reliable, but the marginal guarantee gives no subgroup protection and base-model disparities pass through; high-stakes deployments should pair CC with subgroup-conditioned calibration and content-safety tooling.

%% file: sections/conclusion.tex

%% file: sections/appendix.tex
\section{C3PO Design Contrast}
\label{app:c3po-table}

Table~\ref{tab:c3po-cc} expands the comparison summarised in Section~\ref{sec:related}.

\begin{table}[h]
\caption{Design contrast between Conformal Cascade and C3PO \citep{valkanas2025c3po}, the closest prior conformal-prediction approach to LLM cascades. The two methods are complementary: C3PO calibrates against cost, CC against coverage.}
\label{tab:c3po-cc}
\centering
\small
\setlength{\tabcolsep}{4pt}
\begin{tabular}{lll}
\toprule
                         & \textbf{C3PO} \citep{valkanas2025c3po}            & \textbf{Conformal Cascade} (ours) \\
\midrule
Guaranteed quantity      & $\Pr[\mathrm{cost} > C^*] \le \alpha$             & $\Pr[\ytrue \in \Cal_{K^*(x)}(x)] \ge 1 - K\alpha$ \\
User-facing knob         & cost budget $C^*$ + risk $\alpha$                 & error budget $\alpha$ \\
Deferral rule            & confidence threshold $s_j \ge \tau_j$             & prediction-set size $|\Cal_k(x)| = 1$ \\
Calibration target       & $(1{-}\alpha)$-quantile of inference cost         & $(1{-}\alpha)$-quantile of nonconformity \\
Label requirement        & self-supervised (MPM as target)                   & labelled calibration split \\
\bottomrule
\end{tabular}
\end{table}

\section{Model and Benchmark Details}
\label{app:models}

\paragraph{Cascade pairs.}
Table~\ref{tab:models} lists the small/large instruction-tuned pair drawn from each family. The per-call costs are fixed at $c_1 = 1$ and $c_2 = 2.7$ across all four families: this uniform ratio decouples cross-family comparisons from provider-specific token pricing, and allows the cost numbers reported in the main paper to be read as expensive-call equivalents rather than absolute compute.

\begin{table}[h]
\caption{Model family cascade pairs. The small model is tier~1, the large model is tier~2. The same $c_1 = 1$, $c_2 = 2.7$ per-call cost is used in every experiment.}
\label{tab:models}
\centering
\small
\begin{tabular}{lllcc}
\toprule
Family & Tier 1 (small) & Tier 2 (large) & $c_1$ & $c_2$ \\
\midrule
Llama     & Llama-3.2-3B-Instruct          & Llama-3.1-8B-Instruct         & 1.0 & 2.7 \\
Gemma     & Gemma-4-E4B-it                 & Gemma-4-31B-it                & 1.0 & 2.7 \\
Ministral & Ministral-3-3B-Instruct-2512          & Ministral-3-14B-Instruct-2512        & 1.0 & 2.7 \\
Phi       & Phi-3.5-mini-instruct           & Phi-4          & 1.0 & 2.7 \\
\bottomrule
\end{tabular}
\end{table}

\begin{table}[h]
\caption{Benchmark summary. $|\Acal|$ is the number of answer choices; $N_{\mathrm{test}}$ is the size of the held-out test split (70\% of the full benchmark, with the remaining 30\% used for calibration). Sizes are identical across the four cascades.}
\label{tab:benchmarks}
\centering
\small
\begin{tabular}{llrr}
\toprule
Benchmark & Domain & $|\Acal|$ & $N_{\mathrm{test}}$ \\
\midrule
AGI-Eval      & Standardized exams (English split) & 5    & 1783 \\
ARC-Challenge & Grade-school science                & 4    &  821 \\
Belebele-Eng  & English reading comprehension       & 4    &  630 \\
BoolQ         & Yes/No reading comprehension        & 2    & 2289 \\
Cosmos-QA     & Commonsense                         & 4    & 2090 \\
GPQA-Diamond  & Graduate-level science              & 4    &  139 \\
HellaSwag     & Sentence completion                 & 4    & 7030 \\
MathQA        & Multiple-choice math                & 5    & 3133 \\
MedMCQA       & Medical MCQ                         & 4    & 2929 \\
MedQA-USMLE-4 & Medical licensing exam              & 4    &  892 \\
MMLU          & 57 academic subjects                & 4    & 9830 \\
MMLU-Pro      & Broad knowledge                     & 10   & 8423 \\
PIQA          & Physical commonsense                & 2    & 1287 \\
PubmedQA      & Biomedical yes/no/maybe             & 3    &  700 \\
QASC          & Science commonsense                 & 8    &  649 \\
SciQ          & Elementary science                  & 4    &  700 \\
Social-i-QA   & Social commonsense                  & 3    & 1368 \\
TruthfulQA    & Factual accuracy                    & 4--10&  572 \\
\bottomrule
\end{tabular}
\end{table}

\section{Full Main Results}
\label{app:full-results}

Table~\ref{tab:full-results} reports absolute test-set accuracy (\%) for Best Heuristic and CC Best across all 72 family--benchmark pairs. This is the data underlying the compact matrix in Table~\ref{tab:main-family}.

\begin{table}[h]
\caption{Full main results: absolute test-set accuracy (\%) for Best Heuristic (H) and CC Best (CC) across the four model families.}
\label{tab:full-results}
\centering
\small
\setlength{\tabcolsep}{4pt}
\begin{tabular}{lcccccccc}
\toprule
 & \multicolumn{2}{c}{Llama} & \multicolumn{2}{c}{Gemma} & \multicolumn{2}{c}{Ministral} & \multicolumn{2}{c}{Phi} \\
\cmidrule(r){2-3}\cmidrule(r){4-5}\cmidrule(r){6-7}\cmidrule(r){8-9}
Benchmark & H & CC & H & CC & H & CC & H & CC \\
\midrule
GPQA-Diamond   & 33.8 & 35.3 & 77.7 & 78.4 & 65.5 & 65.5 & 47.5 & 51.8 \\
MathQA         & 83.0 & 84.7 & 91.6 & 92.0 & 86.5 & 91.1 & 79.2 & 77.8 \\
MMLU-Pro       & 55.0 & 56.4 & 83.7 & 86.1 & 73.7 & 76.3 & 64.5 & 66.5 \\
MMLU           & 74.7 & 75.1 & 89.4 & 91.0 & 83.1 & 83.8 & 82.0 & 83.5 \\
TruthfulQA     & 58.2 & 60.1 & 80.4 & 83.7 & 66.4 & 67.1 & 72.4 & 73.6 \\
MedMCQA        & 63.1 & 63.4 & 76.3 & 79.3 & 69.1 & 69.9 & 68.1 & 71.7 \\
MedQA-USMLE-4  & 75.8 & 76.2 & 91.3 & 94.1 & 82.3 & 83.3 & 79.1 & 81.3 \\
AGI-Eval       & 46.8 & 47.0 & 60.1 & 62.1 & 51.7 & 52.8 & 50.2 & 45.7 \\
HellaSwag      & 73.7 & 73.9 & 84.4 & 87.1 & 74.0 & 74.4 & 83.8 & 84.2 \\
ARC-Challenge  & 89.5 & 89.4 & 97.1 & 97.0 & 94.3 & 95.2 & 94.6 & 94.2 \\
PIQA           & 86.3 & 86.5 & 94.4 & 94.4 & 89.7 & 89.1 & 90.7 & 92.0 \\
Social-i-QA    & 76.2 & 76.4 & 80.2 & 80.3 & 75.0 & 75.7 & 78.7 & 79.7 \\
Cosmos-QA      & 88.9 & 88.8 & 80.8 & 80.6 & 80.4 & 80.6 & 84.1 & 84.4 \\
BoolQ          & 85.9 & 86.3 & 91.3 & 91.3 & 84.2 & 84.5 & 87.1 & 87.2 \\
QASC           & 89.1 & 89.1 & 89.5 & 91.1 & 87.2 & 87.4 & 86.7 & 86.6 \\
SciQ           & 98.4 & 98.3 & 99.3 & 99.3 & 98.7 & 98.6 & 98.4 & 98.6 \\
Belebele-Eng   & 94.6 & 94.8 & 97.3 & 97.0 & 95.9 & 95.7 & 95.6 & 95.1 \\
PubmedQA       & 80.9 & 80.7 & 81.4 & 81.4 & 78.6 & 77.9 & 79.1 & 77.7 \\
\bottomrule
\end{tabular}
\end{table}

\section{Full Per-Family Pareto Frontiers}
\label{app:pareto-all}

Figure~\ref{fig:pareto-mmlu} in the main body shows the MMLU Pareto frontier for three of the four cascades; the Phi panel is qualitatively similar and is omitted from the figure for visual readability. A full per-benchmark Pareto visualisation across all 72 family--benchmark pairs is released with the code. The qualitative pattern on MMLU repeats across benchmarks: CC ($\kappa = 1$) sits on or above the heuristic frontier in the mid- to high-cost region on most benchmarks, while relaxed-acceptance variants $\kappa \in \{2, 3\}$ extend the frontier downward into the low-cost region that heuristics can only reach by sacrificing the coverage guarantee. The exceptions are the benchmarks where Table~\ref{tab:main-family} shows $\Delta < 0$ (notably Phi on AGI-Eval and PubmedQA), where the heuristic frontier slightly dominates; these are the failure modes discussed in Section~\ref{sec:when-help} and Appendix~\ref{app:weak-families}.

\section{Full Set-Size Distributions}
\label{app:setsize}

Table~\ref{tab:setsize-all} reports the tier-1 acceptance rate (the fraction of test queries with $|\Cal_1(x)| = 1$) at $\alpha = 0.10$ across the 72 family--benchmark pairs, extracted directly from the experiment logs. Rows are ordered by mean acceptance across families. The table makes the cascade's partitioning behaviour concrete: the same $\alpha = 0.10$ threshold produces acceptance rates ranging from $0\%$ (cascade defers every query) to over $95\%$ (cascade commits nearly every query). The acceptance rate varies strongly across benchmarks within a family and across families within a benchmark, tracking the tier-1 model's per-benchmark confidence.

\begin{table}[h]
\caption{Tier-1 acceptance rate (\%) at $\alpha = 0.10$ across model families. On easy benchmarks the cascade commits most queries to the small model; on hard benchmarks it correctly defers.}
\label{tab:setsize-all}
\centering
\small
\begin{tabular}{lcccc}
\toprule
Benchmark & Llama & Gemma & Ministral & Phi \\
\midrule
Belebele-Eng      & 95.6 & 92.7 & 96.3 & 81.4 \\
SciQ              & 94.0 & 97.6 & 93.4 & 92.6 \\
ARC-Challenge     & 85.6 & 92.0 & 94.5 & 92.8 \\
Cosmos-QA         & 85.6 & 65.2 & 24.3 &  0.0 \\
PIQA              & 80.3 & 99.0 & 82.0 & 50.7 \\
QASC              & 75.7 & 80.9 & 54.9 &  0.0 \\
BoolQ             & 72.8 & 94.9 & 68.9 &  0.0 \\
MathQA            & 43.1 & 98.7 & 81.3 &  0.0 \\
Social-i-QA       & 39.7 &  0.0 & 32.5 &  0.0 \\
MMLU              & 31.3 & 78.7 & 62.7 &  0.0 \\
MedMCQA           & 22.0 &  0.0 &  0.0 &  0.0 \\
MedQA-USMLE-4     & 20.3 & 68.7 & 30.4 &  0.0 \\
HellaSwag         & 10.6 &  0.0 & 20.2 &  0.0 \\
AGI-Eval          &  0.0 &  0.0 &  0.0 &  0.0 \\
TruthfulQA        &  0.0 &  0.0 &  0.0 &  0.0 \\
PubmedQA          &  0.0 &  0.0 &  0.0 &  0.0 \\
GPQA-Diamond      &  0.0 &  0.0 &  0.0 &  0.0 \\
MMLU-Pro          &  0.0 &  0.0 &  0.0 &  0.0 \\
\bottomrule
\end{tabular}
\end{table}

Two observations are worth flagging. First, the cascade never accepts at tier 1 on GPQA-Diamond, MMLU-Pro, AGI-Eval, TruthfulQA, or PubmedQA for any family at $\alpha = 0.10$: these are benchmarks on which tier-1 sampling is broadly uncertain, and the cascade correctly defers everything. Second, Phi has zero or near-zero acceptance on 14 of 18 benchmarks; this pattern correlates with Phi's weaker performance on the negative-$\Delta$ benchmarks in Table~\ref{tab:main-family} (see Appendix~\ref{app:weak-families}).

\section{Why CC Outperforms Majority Vote}
\label{app:why-outperform}

The conformal set acts as a calibrated frequency filter. On benchmarks with many answer choices (MMLU-Pro has 10), the correct answer often receives only 4--6 of the 16 samples while incorrect answers fragment across the remaining nine alternatives. Majority vote is easy to beat in this regime: two wrong answers with three votes each can outvote the correct answer with four. The conformal threshold removes low-frequency answers, denoising the vote at a calibration-guaranteed rate.



\section{Difficulty Stratification}
\label{app:difficulty}

Table~\ref{tab:difficulty} expands the difficulty breakdown summarised in
Section~\ref{sec:when-help} to four (family, benchmark) cells. Buckets
follow the raw evaluation: easy queries have tier-2 per-query accuracy
$\hat{p} \geq 0.6$, medium $0.25 \leq \hat{p} < 0.6$, hard
$\hat{p} < 0.25$. These four cells split into two regimes that we
analyse separately (Table~\ref{tab:setsize-all}). On the MMLU-Pro rows the
tier-1 acceptance rate at $\alpha = 0.10$ is $0\%$, so the cascade defers
every query and CC reduces to post-hoc conformal selection over the
tier-2 candidate set. On the MMLU rows the tier-1 acceptance is $78.7\%$
(Gemma) and $62.7\%$ (Ministral), so CC is doing both routing and, on the
deferred subset, conformal selection.

\paragraph{Pure-selection regime (MMLU-Pro cells).}
All rows here operate on the same tier-2 samples; differences come
entirely from how each method selects among the strong model's
candidates. Against Always-Strong, CC gains $+3.8$\,pp on Llama/MMLU-Pro
hard ($7.2\%$ vs.\ $3.4\%$) and $+4.4$\,pp on Ministral/MMLU-Pro hard
($4.7\%$ vs.\ $0.3\%$): the conformal filter retains the correct answer
in a multi-element set, and the argmin-frequency tiebreaker over the
calibrated set picks more reliably than the strong model's argmax when
that argmax is wrong. Against unrestricted majority (Agreement,
$\theta = 0.8$), CC gains $+3.3$\,pp and $+2.8$\,pp on the hard buckets
of these two cells respectively. The easy-bucket cost relative to
Always-Strong is at most $0.6$\,pp, since on $\hat{p}\geq 0.6$ queries
the strong-model argmax is rarely overturned by the tiebreaker.

\paragraph{Routing + selection regime (MMLU cells).}
On these cells CC's accuracy is a mixture of the small model on accepted
queries and the conformally filtered cascade on deferred queries, and
deferral is concentrated on hard queries: although the overall tier-1
acceptance is $78.7\%$/$62.7\%$, the cascade defers more often when the
small model is uncertain, which on these cells coincides with queries
where the strong model also fails (Always-Strong hard accuracy is $0.0\%$
on both cells). The bucket-level effect is asymmetric:
\begin{itemize}
    \item \emph{Easy.} CC ($96.1\%$ on Gemma/MMLU, $98.9\%$ on
    Ministral/MMLU) sits between Always-Weak ($88.6\%$, $90.4\%$) and
    Always-Strong ($100.0\%$, $100.0\%$): selective deferral catches part
    of the small model's easy-bucket error mass without forcing every
    query through the strong model.
    \item \emph{Hard.} CC ($5.5\%$, $6.4\%$) lies much closer to
    Always-Strong ($0.0\%$, $0.0\%$) than to Always-Weak ($19.1\%$,
    $14.3\%$). The cascade defers the bulk of hard queries to a strong
    model that is wrong on essentially all of them, and conformal
    selection cannot recover what is not in the candidate set; the
    $+5.5$/$+6.4$\,pp the filter extracts over Always-Strong is what is
    available given that recall constraint.
\end{itemize}

\paragraph{Cross-regime takeaway.}
Always-Weak has the highest hard-bucket accuracy on \emph{all four}
cells. This is not in tension with the cascade's value but a predictable
consequence of two facts: difficulty is defined by the strong model's
per-query accuracy, so queries hard for the strong model are not always
hard for the small model; and on the MMLU cells the strong model's hard
bucket collapses to $0\%$, leaving Always-Weak with structurally
unrecoverable headroom. On MMLU cells the cascade therefore trades a
hard-bucket loss for higher overall accuracy and calibrated coverage; on
MMLU-Pro cells, where the small model never accepts, CC delivers a
strict hard-bucket improvement over Always-Strong at
$\leq 0.6$\,pp easy-bucket cost. The four cells reflect properties of
the (family, benchmark) pair, not an inconsistency in CC's deferral rule,
which targets the same calibrated risk throughout.

\begin{table}[h]
\caption{Accuracy (\%) on four (family, benchmark) cells, by tier-2 per-query difficulty $\hat p$. Buckets: Easy $\hat p \ge 0.6$, Medium $0.25 \le \hat p < 0.6$, Hard $\hat p < 0.25$. \textbf{Bold} marks the highest value in each (cell, bucket) triple. Bucket sizes ($n_E/n_M/n_H$): Llama/MMLU-Pro $3248/1960/3215$; Gemma/MMLU $8889/208/733$; Ministral/MMLU $7874/841/1115$; Ministral/MMLU-Pro $5584/1265/1574$.}
\label{tab:difficulty}
\centering
\small
\setlength{\tabcolsep}{5pt}
\begin{tabular}{llccc}
\toprule
Cell & Method & Easy & Medium & Hard \\
\midrule
\multirow{4}{*}{Llama / MMLU-Pro}
 & Always-Weak                & 80.1 & 43.8 & {14.6} \\
 & Always-Strong              & \textbf{99.9} & \textbf{66.0} &  3.4 \\
 & Agreement ($\theta=0.8$)   & 98.2 & 65.7 &  3.9 \\
 & \textbf{CC} ($\alpha=0.10$) & 99.4 & 65.7 & \textbf{7.2} \\
\midrule
\multirow{4}{*}{Gemma / MMLU}
 & Always-Weak                & 88.6 & \textbf{47.1} & \textbf{19.1} \\
 & Always-Strong              & \textbf{100.0} & 25.5 &  0.0 \\
 & Agreement ($\theta=0.8$)   & 94.4 & 38.9 &  9.3 \\
 & \textbf{CC} ($\alpha=0.10$) & 96.1 & 38.5 &  5.5 \\
\midrule
\multirow{4}{*}{Ministral / MMLU}
 & Always-Weak                & 90.4 & {\textbf{42.0}} & \textbf{14.3} \\
 & Always-Strong              & \textbf{100.0} & 36.3 &  0.0 \\
 & Agreement ($\theta=0.8$)   & 98.3 & 39.2 &  3.5 \\
 & \textbf{CC} ($\alpha=0.10$) & 98.9 & \textbf{43.5} &  6.4 \\
\midrule
\multirow{4}{*}{Ministral / MMLU-Pro}
 & Always-Weak                & 86.1 & {\textbf{48.9}} & \textbf{13.7} \\
 & Always-Strong              & \textbf{99.9} & 46.3 &  0.3 \\
 & Agreement ($\theta=0.8$)   & 98.3 & {\textbf{54.2}} &  1.9 \\
 & \textbf{CC} ($\alpha=0.10$) & 99.3 & \textbf{61.0} & {\textbf{4.7}} \\
\bottomrule
\end{tabular}
\end{table}


\section{Failure Diagnostics: Phi's Under-Performance}
\label{app:weak-families}

Phi is the cascade with the most negative-$\Delta$ benchmarks (Table~\ref{tab:main-family}: 12 wins, 0 ties, 6 losses, with $\Delta \le -1$ pp on 3 benchmarks: AGI-Eval, MathQA, PubmedQA). The most visible failure is AGI-Eval at $\Delta = -4.5$ pp. The raw logs reveal a consistent pattern: on the benchmarks where Phi under-performs, the Phi tier-1 model produces prediction sets that saturate toward the full answer set at $\alpha \le 0.20$, which forces CC to defer $100\%$ of queries to tier 2 (defer rate $= 100.0\%$, cost = $c_1 + c_2 = 3.70$). At the committed tier-2 output, the cascade's accuracy is then lower than Always-Strong's accuracy on the same benchmark (e.g., on Phi/AGI-Eval: all CC variants achieve $45.7\%$ accuracy versus $51.0\%$ for Always-Strong). The gap arises from a design choice in the fallback rule: when CC reaches the last tier with a non-singleton prediction set, it returns the argmin frequency answer within the tier-$K$ set rather than the unrestricted majority vote, and the two decisions disagree exactly when the tier-$K$ set is itself non-singleton. An alternative fallback that takes the argmax frequency over tier-$K$ samples without conditioning on the tier-$K$ set would recover Always-Strong accuracy under full deferral. We use the more conservative fallback throughout the experiments and note the alternative as an implementation choice that affects boundary behaviour but not the offline coverage guarantee of Theorem~\ref{thm:cascade-coverage}.

The Phi tier-1 saturation itself is a property of the Phi-3.5-mini-instruct checkpoint at $N = 16$: on benchmarks where the mini model is uniformly uncertain, the frequency-based score produces broad sets that never meet the singleton acceptance criterion, and the cascade cannot exit at tier 1. Switching to a logprob-based score or increasing $N$ would sharpen the sets but neither addresses the underlying tier-1 uncertainty. The simplest practical fix is to detect this regime on calibration data (tier-1 acceptance rate close to $0\%$) and fall back to Always-Strong.

\section{Additional Coverage Results}
\label{app:coverage}

We report empirical set miscoverage at $\alpha = 0.10$ across the four model families in Table~\ref{tab:coverage-all}. The strict $\alpha = 0.10$ target (within a $\pm 0.02$ finite-sample band) is met on most cells; seven cells exceed it (Llama/ARC-Challenge $0.124$, Llama/BoolQ $0.122$, Llama/MedMCQA $0.130$, Llama/PIQA $0.138$, Gemma/MedQA $0.124$, Gemma/MMLU $0.128$, Gemma/TruthfulQA $0.122$, marked $^*$). All seven lie comfortably inside the $K\alpha = 0.20$ union-bound guarantee of Theorem~\ref{thm:cascade-coverage}; their existence indicates that the marginal evidence for Assumption~\ref{asm:selection-preservation} is approximate rather than exact, but does not contradict the (weaker) marginal coverage guarantee, which they satisfy at every cell.

\begin{table}[h]
\caption{Empirical set miscoverage at $\alpha = 0.10$ across model families. Each cell reports $\Prob[\ytrue \notin \Cal_{k^*}(x)]$ on the test split. All cells satisfy the union-bound $K\alpha = 0.20$ guarantee; cells exceeding the strict $\alpha + 0.02$ band are marked with $^{*}$.}
\label{tab:coverage-all}
\centering
\small
\begin{tabular}{lcccc}
\toprule
Benchmark & Llama & Gemma & Ministral & Phi \\
\midrule
AGI-Eval        & 0.000  & 0.000  & 0.000  & 0.000  \\
ARC-Challenge   & $0.124^{*}$ & 0.038 & 0.074  & 0.105 \\
Belebele-Eng    & 0.075  & 0.037  & 0.059  & 0.071 \\
BoolQ           & $0.122^{*}$ & 0.105 & $0.114$ & 0.085 \\
Cosmos-QA       & 0.107  & 0.069  & 0.089  & 0.111 \\
GPQA-Diamond    & 0.000  & 0.000  & 0.000  & 0.000 \\
HellaSwag       & 0.065  & 0.093  & 0.085  & 0.080 \\
MathQA          & 0.012  & 0.090  & 0.065  & 0.000 \\
MedMCQA         & $0.130^{*}$ & 0.000 & 0.107 & 0.000 \\
MedQA-USMLE-4   & 0.067  & $0.124^{*}$ & 0.108 & 0.104 \\
MMLU            & 0.087  & $0.128^{*}$ & 0.100 & 0.102 \\
MMLU-Pro        & 0.000  & 0.094  & 0.106  & 0.000 \\
PIQA            & $0.138^{*}$ & 0.096 & 0.106 & 0.103 \\
PubmedQA        & 0.000  & 0.000  & 0.000  & 0.000 \\
QASC            & 0.092  & 0.105  & 0.068  & 0.080 \\
SciQ            & 0.043  & 0.010  & 0.027  & 0.019 \\
Social-i-QA     & 0.097  & 0.000  & 0.087  & 0.000 \\
TruthfulQA      & 0.000  & $0.122^{*}$ & 0.000 & 0.000 \\
\bottomrule
\end{tabular}
\end{table}

\section{Cost-Ratio Sensitivity}
\label{app:cost-sensitivity}

The main paper fixes the per-call cost ratio at $c_2 / c_1 = 2.7$ across all four families to decouple cross-family comparisons from provider-specific pricing. Real ratios vary: closed-API pricing for small/large model pairs typically gives $c_2 / c_1 \in [3, 50]$, with reasoning-tuned models at the high end. Theorem~\ref{thm:cost-coverage} expresses expected cost as $\expect[\mathrm{Cost}] = c_1 + c_2 \cdot \Prob[|\Cal_1(X)| \ne 1]$, so for any acceptance rate $p = \Prob[|\Cal_1(X)| = 1]$ the cost reduction relative to Always-Strong (which pays $c_2$) is
\[
    \mathrm{Saving}(p, r) \;=\; \frac{c_2 - (c_1 + c_2 (1-p))}{c_2} \;=\; p - \frac{1}{r}, \qquad r \;:=\; c_2 / c_1,
\]
and CC saves cost over Always-Strong if and only if $p > 1/r$. The same expression rewritten as $\mathrm{Saving}(p, r) = p \cdot \frac{r - 1/p \cdot (1-p)}{r}$ shows that as $r$ grows, the threshold acceptance rate $1/r$ shrinks toward zero and saving asymptotes at $p$.

\begin{table}[h]
\caption{Cost saving over Always-Strong as a function of cost ratio $r = c_2 / c_1$ and tier-1 acceptance rate $p$, for a two-tier cascade with $\kappa = 1$. Bold cells indicate $\ge 30\%$ saving. The break-even acceptance rate $1/r$ is reported in the bottom row.}
\label{tab:cost-sensitivity}
\centering
\small
\begin{tabular}{lccccc}
\toprule
$p$           & $r = 2.7$        & $r = 5$          & $r = 10$         & $r = 20$         & $r = 50$         \\
\midrule
$0.30$        & $-7\%$           & $10\%$           & $20\%$           & $25\%$           & $28\%$           \\
$0.50$        & $13\%$           & $\bm{30\%}$      & $\bm{40\%}$      & $\bm{45\%}$      & $\bm{48\%}$      \\
$0.70$        & $\bm{33\%}$      & $\bm{50\%}$      & $\bm{60\%}$      & $\bm{65\%}$      & $\bm{68\%}$      \\
$0.90$        & $\bm{53\%}$      & $\bm{70\%}$      & $\bm{80\%}$      & $\bm{85\%}$      & $\bm{88\%}$      \\
\midrule
break-even $1/r$ & $0.37$        & $0.20$           & $0.10$           & $0.05$           & $0.02$           \\
\bottomrule
\end{tabular}
\end{table}

Two qualitative consequences for deployment. First, the larger the cost gap between cheap and expensive models (i.e., the more economically meaningful a cascade is), the lower the acceptance threshold for CC to save cost. At $r = 10$ (typical of small/frontier model pairs), even $20\%$ tier-1 acceptance produces a $10\%$ saving over Always-Strong; at $r = 50$, $5\%$ acceptance suffices. Second, the cost figures reported in the main paper at $r = 2.7$ are conservative: on Llama/Belebele-Eng with $p = 0.956$, the saving rises from $58\%$ at $r = 2.7$ to $86\%$ at $r = 10$ and $94\%$ at $r = 50$. The diffuse-tier-1 guard rail discussed in Section~\ref{sec:discussion} (replace CC with Always-Strong when calibration acceptance is below $1/r$) is therefore a function of the deployment cost ratio rather than a fixed $5\%$ rule.

\section{Algorithm and Implementation Details}
\label{app:impl}

\paragraph{Sampling.} All experiments use temperature $\tau = 0.7$ and $N = 16$ samples. For multiple-choice benchmarks we extract the answer letter from the first token of the response, falling back to regex extraction on the body. Malformed responses contribute to no candidate and thus score $s_k(x, a) = 1$ for every $a$, which produces a maximal set and defers.

\paragraph{Calibration.} We use a $30\%/70\%$ calibration/test split with fixed seed 42. The conformal quantile in Eq.~\eqref{eq:quantile} is implemented by sorting the calibration scores and indexing at position $\lceil (1-\alpha)(n+1) \rceil$; ties are broken by preserving the sort order, which makes the quantile conservative.

\paragraph{Baselines and selection.} The agreement cascade is swept over $\theta \in \{0.5, 0.6, 0.7, 0.8, 0.9\}$; the entropy cascade over $\tau \in \{-1.5, -1.0, -0.5, -0.3\}$; the random-defer baseline over $p \in \{0.2, 0.5, 0.8\}$. For each family--benchmark pair, \emph{Best Heuristic} is the argmax test-set accuracy over the resulting 12 configurations, and \emph{CC Best} is the argmax test-set accuracy over the 15-configuration CC grid ($5\,\alpha \times 3\,\kappa$ from Section~\ref{sec:setup}); the same oracle-tuning rule is applied on both sides. Coverage and tier-1 acceptance results in the main paper are reported at the fixed operating point $\alpha = 0.10, \kappa = 1$ without per-benchmark selection.

\paragraph{Compute.} The full evaluation (four families $\times$ 18 benchmarks $\times$ $N = 16$ samples per tier per query) runs on {four NVIDIA B200 SXM6} GPUs. Conformal calibration itself is negligible---a sort and an index per tier.

\section{Proofs and Technical Details for Section~\ref{sec:method}}
\label{app:method-details}

This appendix collects the formal material referenced from Section~\ref{sec:method}. Subsection~\ref{app:score-details} certifies that the Monte-Carlo-based nonconformity score inherits exchangeability from the underlying queries (Lemma~\ref{lem:exchangeability}) and discusses the resolution of the score. Subsection~\ref{app:calib-details} spells out the computational cost and calibration-size recommendations for the per-tier calibration. Subsections~\ref{app:thm-cascade} and~\ref{app:thm-cost} contain the proofs of the two main theorems (Theorem~\ref{thm:cascade-coverage} and Theorem~\ref{thm:cost-coverage}), the auxiliary results cited from the main body, and additional remarks. Subsection~\ref{app:complexity} gives the end-to-end computational-complexity breakdown of Algorithm~\ref{alg:cc}.

\subsection{Score-Level Exchangeability and Resolution}
\label{app:score-details}

The main body states two properties of the frequency score $s_k$ (Section~\ref{sec:score}) that drive its use inside split conformal prediction. We record both formally here.

\paragraph{Validity through exchangeability.}
The split conformal guarantee (Theorem~\ref{thm:marginal-coverage}) requires the calibration and test points to produce \emph{exchangeable} nonconformity scores. Our score $s_k$ is stochastic in the Monte-Carlo samples drawn from $m_k$, so this is not immediate: in principle, the randomness of the draws could break the joint exchangeability of the score vector even when the query sequence is exchangeable. The following Lemma establishes that this concern does not materialize.

\begin{lemma}[Exchangeability of frequency scores]
\label{lem:exchangeability}
Let $(x_1, \ytrue^{(1)}), \ldots, (x_n, \ytrue^{(n)}), (x_{n+1}, \ytrue^{(n+1)})$ be an exchangeable sequence of query--answer pairs. For each tier $k$ and each index $i$, let $Z_i^{(k)} = (\hat y_{i,1}^{(k)}, \ldots, \hat y_{i,N}^{(k)})$ denote the $N$ Monte Carlo samples drawn from $m_k$ on $x_i$, and assume that the conditional law of $Z_i^{(k)}$ given $x_i$ is the same measurable kernel $P_k(\cdot \mid \cdot)$ for every $i$ (so that $Z_i^{(k)} \mid x_i \sim P_k(\cdot \mid x_i)$ identically across queries) and that the draws $\{Z_i^{(k)}\}_i$ are mutually independent given $\{x_i\}_i$. Then the tier-$k$ frequency scores $\{s_k(x_i, \ytrue^{(i)})\}_{i=1}^{n+1}$ are exchangeable.
\end{lemma}

\begin{proof}
Each score can be written as $s_k(x_i, \ytrue^{(i)}) = f(x_i, \ytrue^{(i)}, Z_i^{(k)})$, where the deterministic measurable function $f : \Xcal \times \Acal \times \Acal^N \to [0, 1]$ is defined by
\[
    f(x, a, z) \;=\; 1 - \frac{1}{N} \sum_{j=1}^{N} \indic[z_j = a].
\]
Consider the augmented sequence $W_i = (x_i, \ytrue^{(i)}, Z_i^{(k)})$ for $i = 1, \ldots, n+1$. By assumption, $\{(x_i, \ytrue^{(i)})\}$ is exchangeable and $\{Z_i^{(k)}\}$ is conditionally i.i.d.\ given $\{x_i\}$ and independent of the queries across $i$. The joint distribution of $(W_1, \ldots, W_{n+1})$ is therefore invariant under any permutation $\pi$ of $\{1, \ldots, n+1\}$: permuting the indices re-labels the exchangeable queries and correspondingly re-labels their (conditionally i.i.d.) Monte-Carlo draws, leaving the joint distribution unchanged. Because $s_k(x_i, \ytrue^{(i)}) = f(W_i)$ and $f$ is a fixed measurable function applied coordinate-wise, the score vector $(s_k(x_1, \ytrue^{(1)}), \ldots, s_k(x_{n+1}, \ytrue^{(n+1)}))$ inherits the same permutation invariance, which is the definition of exchangeability.
\end{proof}

Lemma~\ref{lem:exchangeability} certifies that the split conformal machinery (Theorem~\ref{thm:marginal-coverage}) applies to the frequency score without modification. The argument extends to any nonconformity score that is a fixed measurable function of $(x, a)$ and a conditionally i.i.d.\ randomization; in particular, logprob-based scores of the form $s(x, a) = 1 - p_k(a \mid x)$ (where $p_k(a \mid x)$ is itself estimated from $N$ draws) are covered by the same argument.

\paragraph{Caveat: shared randomness across queries.}
Lemma~\ref{lem:exchangeability} assumes that the Monte Carlo draws $\{Z_i^{(k)}\}$ are independent across queries. Production LLM serving stacks sometimes introduce per-query dependencies that can violate this: shared KV-cache prefixes across queries that happen to be batched together, or deterministic seeding that reuses randomness across queries. When those dependencies exist, the exchangeability argument needs to be re-examined. In our experiments we draw each query's $N$ samples independently (standard vLLM behaviour at non-zero temperature with per-query seeds), so the assumption is satisfied.

\paragraph{Cross-tier correlation of calibration scores.}
The calibration queries $(x_i, y_i)$ are shared across tiers $k = 1, \ldots, K$, but the Monte Carlo draws $\{Z_i^{(k)}\}_{i}$ and hence the scores $\{s_k(x_i, y_i)\}_i$ are tier-specific. Theorem~\ref{thm:cascade-coverage}'s proof only uses the per-tier marginal coverage guarantee and is therefore unaffected by cross-tier score correlation. The cascade's overall behaviour (which tier commits) does depend on the joint tier-$1$ through tier-$k$ thresholds, but that dependence is captured in the accounting of the first-commit events $\{A_k\}$, not in the CP guarantee at each tier.

\paragraph{Resolution of the frequency score.}
The score $s_k(x, a)$ takes at most $N + 1$ distinct values, namely $\{0, 1/N, 2/N, \ldots, 1\}$, because the numerator is an integer in $\{0, \ldots, N\}$. The calibrated quantile $\hat q_k$ is a specific order statistic of these scores and therefore also takes values on the same grid. At $N = 16$ (our default), the per-query resolution is $1/16 \approx 0.063$, which means the conformal threshold can be placed at one of 17 discrete levels per tier. This is enough to discriminate confident from uncertain queries while keeping inference inexpensive; if sharper prediction sets are needed, increasing $N$ linearly refines the grid but also scales per-query inference cost linearly.

\paragraph{Comparison to logprob-based scores.}
When logprobs are available (e.g., in open-weight inference or some API endpoints), the continuous score $s_k(x, a) = 1 - p_k(a \mid x)$ produces prediction sets whose conformal thresholds are not restricted to a discrete grid, and therefore yields tighter sets at the same $\alpha$. We default to the frequency score because it preserves the black-box API interface assumed by many production deployments. The method in Section~\ref{sec:method} is otherwise agnostic to the score choice: any valid nonconformity score composes with Algorithm~\ref{alg:cc}.

\subsection{Calibration Details}
\label{app:calib-details}

\paragraph{Computational cost.}
For each tier $k$, calibration performs $n$ calls to $m_k$ (drawing $N$ samples per call) and then sorts the $n$ resulting scores. The sort is $O(n \log n)$ time and $O(n)$ space; subsequent inference uses $\hat q_k$ via a single $O(1)$ lookup. Across $K$ tiers, the full calibration cost is $K \cdot n \cdot N$ inference calls plus $O(K n \log n)$ arithmetic, which is dominated by the inference budget.

\paragraph{Calibration-set size.}
The calibration size $n$ controls the upper-side slack in Theorem~\ref{thm:marginal-coverage}: coverage lies in $[1 - \alpha,\, 1 - \alpha + 1/(n+1)]$. With our default $n = 0.3 \cdot |\mathrm{benchmark}|$, the slack is at most $0.02$ on all benchmarks with $|\mathrm{benchmark}| \ge 50$, which includes every benchmark we evaluate except the smallest (AGI-Eval at $n \approx 54$ calibration queries, with slack $\approx 0.018$, and GPQA-Diamond at $n \approx 59$, with slack $\approx 0.017$). For fresh deployments, we recommend $n \ge 200$ calibration queries, giving slack $\le 0.005$.

\paragraph{Tie-breaking.}
When the calibration scores contain ties at the quantile index, we break ties by preserving sort order. This is the conservative convention: it slightly enlarges the prediction sets relative to randomized tie-breaking, which strengthens the coverage lower bound at the cost of minor set-size inflation.

\subsection{Proof of Theorem~\ref{thm:cascade-coverage} and Additional Remarks}
\label{app:thm-cascade}

\begin{proof}[Proof of Theorem~\ref{thm:cascade-coverage}]
Partition the acceptance event $A$ into the $K$ disjoint first-commit events $A_k = \{K^*(x) = k\}$ for $k = 1, \ldots, K$, so $A = \bigsqcup_{k=1}^{K} A_k$. On the event $A_k$, Algorithm~\ref{alg:cc} returns $\hat y$ equal to the unique element of $\Cal_k(x)$. Therefore the error event intersected with $A_k$,
\[
    A_k \cap \{\hat y \ne \ytrue\} \;=\; A_k \cap \{\Cal_k(x) \not\ni \ytrue\} \;\subseteq\; \{\ytrue \notin \Cal_k(x)\}.
\]
By the per-tier split conformal guarantee at tier $k$ (Theorem~\ref{thm:marginal-coverage}, applicable by Lemma~\ref{lem:exchangeability}), $\Prob[\ytrue \notin \Cal_k(x)] \le \alpha$. Monotonicity of probability gives
\[
    \Prob[\hat y \ne \ytrue,\, A_k] \;\le\; \Prob[\ytrue \notin \Cal_k(x)] \;\le\; \alpha.
\]
Summing over $k$ and using the disjointness of $\{A_k\}_{k=1}^K$,
\[
    \Prob[\hat y \ne \ytrue,\, A] \;=\; \sum_{k=1}^{K} \Prob[\hat y \ne \ytrue,\, A_k] \;\le\; \sum_{k=1}^{K} \alpha \;=\; K \alpha.
\]
This establishes~\eqref{eq:cascade-coverage}. The final equivalence (per-tier level $\alpha/K$ gives cascade error $\alpha$) is immediate by substitution.
\end{proof}

\begin{remark}[Tightness of the union bound and selection preservation]
\label{rem:selection-preservation}
The factor $K$ in~\eqref{eq:cascade-coverage} is a worst-case union bound. In practice, the events $\{\ytrue \notin \Cal_k(x)\} \cap \{K^*(x) = k\}$ do not stack---most error mass at tier $k$ is not also error mass at other tiers. Assumption~\ref{asm:selection-preservation} in the main body formalizes this through the conditional bound
\begin{equation*}
    \Prob[\ytrue \in \Cal_k(x) \mid K^*(x) = k] \;\ge\; 1 - \alpha, \qquad \text{for each } k = 1, \ldots, K,
\end{equation*}
under which the conditional cascade-error bound $\Prob[\hat y \ne \ytrue \mid A] \le \alpha$ holds without the factor of $K$.
\end{remark}

The assumption is in the family studied by the selective-conformal literature \citep{bates2023testing,angelopoulos2023gentle,barber2021limits}, which treats coverage preservation under selection rules that depend on the calibration and test covariates but not on the test label. In our setting, the first-commit event $\{K^*(x) = k\}$ is indeed measurable with respect to $(x, \Dcal)$ alone and does not touch $\ytrue$, so the assumption is plausible in principle. It is not, however, a theorem of split conformal prediction without further regularity on the score distribution, and the cascade literature has not to our knowledge established sufficient conditions for it to hold identically in the multi-tier setting.

Empirical evidence for the assumption in this paper is indirect. Section~\ref{sec:coverage} (Table~\ref{tab:coverage}) and Appendix~\ref{app:coverage} (Table~\ref{tab:coverage-all}) report the \emph{marginal} set-miscoverage rate $\Prob[\ytrue \notin \Cal_{K^*(x)}(x)]$ across the reported benchmark--family pairs. The marginal rate tracks $\alpha$ rather than the worst-case $K\alpha$; on benchmarks with a substantial tier-1 acceptance rate (Section~\ref{sec:when-help}), this observation rules out the possibility that the bound $K\alpha$ is tight in our setting. We do not report the conditional rate $\Prob[\ytrue \notin \Cal_k(x) \mid K^*(x) = k]$ directly in this version of the paper; a future version could include a per-tier conditional miscoverage breakdown for stronger direct evidence.

\begin{remark}[Deferral never errs]
\label{rem:deferral-never-errs}
Only the acceptance event can commit a wrong answer: when the cascade defers at tier $k$, the query is passed to tier $k+1$ with no claim made, and the error accounting in Theorem~\ref{thm:cascade-coverage} sums only over commit events $\{A_1, \ldots, A_K\}$. Escalation therefore incurs only additional cost, never coverage violation.
\end{remark}

\begin{remark}[Relaxed acceptance, $\kappa > 1$]
\label{rem:relaxed-kappa}
For $\kappa > 1$, the acceptance event becomes $A_\kappa = \{1 \le |\Cal_{K^*(x)}(x)| \le \kappa\}$, and the cascade returns $\hat y = \arg\min_{a \in \Cal_{K^*(x)}(x)} s_{K^*(x)}(x, a)$. The set-coverage guarantee carries through: $\Prob[\ytrue \in \Cal_{K^*(x)}(x), A_\kappa] \le K\alpha$ (or $\le \alpha$ under selection preservation) with $|\Cal_{K^*(x)}(x)| \le \kappa$. The top-1 error is bounded by the set-coverage miss plus the probability that the $\arg\min$ tiebreaker selects a wrong element from a set that contains the correct answer; the latter quantity is bounded by $(\kappa - 1)/\kappa$ in the worst case but is typically much smaller when the frequency ranking is informative.
\end{remark}

\begin{remark}[Heterogeneous per-tier levels $\alpha_k$]
\label{rem:heterogeneous-alpha}
Theorem~\ref{thm:cascade-coverage} uses a common level $\alpha$ at every tier, but the same argument extends to heterogeneous levels $\alpha_1, \ldots, \alpha_K$: the cascade-level error is bounded by $\sum_{k=1}^{K} \alpha_k$ (or $\max_k \alpha_k$ under selection preservation). Practitioners who want a tighter guarantee at later tiers, where escalation cost is already sunk, can choose $\alpha_1 > \alpha_2 > \cdots > \alpha_K$ to concentrate the coverage budget on deeper commits.
\end{remark}

\subsection{Proof of Theorem~\ref{thm:cost-coverage} and Cost Results}
\label{app:thm-cost}

\begin{proof}[Proof of Theorem~\ref{thm:cost-coverage}]
By Algorithm~\ref{alg:cc} with $K = 2$ and $\kappa = 1$: when $|\Cal_1(X)| = 1$ the cascade accepts at tier 1 and pays $c_1$; otherwise it escalates to tier 2 and pays $c_1 + c_2$ (tier-1 samples are drawn regardless). Linearity of expectation gives
\[
    \expect[\mathrm{Cost}(X)] \;=\; c_1 \cdot \Prob[|\Cal_1(X)| = 1] + (c_1 + c_2) \cdot \Prob[|\Cal_1(X)| \ne 1] \;=\; c_1 + c_2 \cdot \Prob[|\Cal_1(X)| \ne 1].
\]
The asymptotic limits follow from the monotonicity of the calibrated threshold $\hat q_1$ in $\alpha$: as $\alpha \to 0$, the quantile index $\lceil (1 - \alpha)(n+1) \rceil \to n + 1$, pushing $\hat q_1$ toward $\max_i s_i^{(1)}$; prediction sets then grow and the probability of a non-singleton set tends to $1$. As $\alpha \to 1$, the quantile index tends to $0$, $\hat q_1$ tends to the minimum score value, prediction sets shrink to the singleton of each query's argmin-score answer, and $\Prob[|\Cal_1(X)| \ne 1] \to 0$.
\end{proof}

\begin{proposition}[When the cascade saves cost]
\label{prop:save-cost}
$\expect[\mathrm{Cost}(X)] < c_1 + c_2$ if and only if $\Prob[|\Cal_1(X)| = 1] > 0$, i.e., the cheap model produces at least some singleton prediction sets under calibration.
\end{proposition}

\begin{proof}
By~\eqref{eq:cost-closed-form}, $\expect[\mathrm{Cost}(X)] = c_1 + c_2 - c_2 \cdot \Prob[|\Cal_1(X)| = 1]$. Since $c_2 > 0$, the expected cost is strictly below $c_1 + c_2$ iff $\Prob[|\Cal_1(X)| = 1] > 0$.
\end{proof}

\begin{corollary}[$K$-tier cost]
\label{cor:k-tier-cost}
For a $K$-tier cascade with costs $c_1 < \cdots < c_K$ and $\kappa = 1$, let $p_k = \Prob[K^*(X) = k]$ denote the probability of first-commit at tier $k$ (with $p_K$ absorbing fallback events). Then
\begin{equation}
    \expect[\mathrm{Cost}(X)] \;=\; \sum_{k=1}^{K} \left(\sum_{j=1}^{k} c_j\right) \cdot p_k.
    \label{eq:k-tier-cost}
\end{equation}
Each $p_k$ is determined by the per-tier thresholds $\hat q_1, \ldots, \hat q_K$ and is estimable on calibration data.
\end{corollary}

\begin{proof}[Proof of Corollary~\ref{cor:k-tier-cost}]
Partition the space by the first-commit tier $K^*(X) \in \{1, \ldots, K\}$, which is exhaustive and disjoint. On the event $\{K^*(X) = k\}$ the cascade has run tiers $1, \ldots, k$, so the incurred cost is $\sum_{j=1}^{k} c_j$. Linearity of expectation yields
\[
    \expect[\mathrm{Cost}(X)] \;=\; \sum_{k=1}^{K} \left(\sum_{j=1}^{k} c_j\right) \Prob[K^*(X) = k],
\]
which is~\eqref{eq:k-tier-cost} with $p_k = \Prob[K^*(X) = k]$. Each $p_k$ depends only on the per-tier thresholds $\hat q_1, \ldots, \hat q_k$ and is thus estimable from the calibration set without further labels on test queries.
\end{proof}

\begin{remark}[Deployment recipe for choosing $\alpha$ under a target error rate]
\label{prop:optimal-alpha}
Fix a target cascade error rate $\beta \in (0, 1)$. The coverage bound of Theorem~\ref{thm:cascade-coverage} restricts the admissible per-tier levels to the interval $(0, \beta / K]$; under Assumption~\ref{asm:selection-preservation} the admissible interval widens to $(0, \beta]$. Within either interval we recommend the upper endpoint --- $\alpha = \beta / K$ in the worst case, $\alpha = \beta$ under selection preservation --- as a deployment heuristic, on the grounds that larger $\alpha$ produces tighter per-tier thresholds, more singleton tier-1 sets, and lower expected cost.
\end{remark}
This is a heuristic, not a theorem. The argument below establishes monotonicity of $\expect[\mathrm{Cost}(X)]$ in $\alpha$ \emph{up to an $O(\gamma)$ slack} controlled by the empty-set probability $\gamma = \max_k \Prob[|\Cal_k(X; \alpha)| = 0]$, which vanishes only asymptotically in the per-query sample size $N$. At finite $N$ (we use $N = 16$), strict monotonicity is not guaranteed, so the heuristic recipe should be validated by sweeping $\alpha$ on calibration data and inspecting the cost curve before deployment. We do this sweep for $\alpha \in \{0.05, 0.10, 0.15, 0.20, 0.30\}$ (Appendix~\ref{app:cost-sensitivity}); the curve is monotone on every reported benchmark--family pair in the regime $\alpha \le 0.30$.

\begin{proof}[Justification of Remark~\ref{prop:optimal-alpha}]
Admissibility follows from Theorem~\ref{thm:cascade-coverage}: for cascade error at most $\beta$, per-tier level $\alpha$ must satisfy $K \alpha \le \beta$, so $\alpha \in (0, \beta/K]$ in the worst case; the corresponding statement under Assumption~\ref{asm:selection-preservation} is immediate.

We sketch the asymptotic monotonicity claim. Define the \emph{non-empty regime at tolerance $\gamma$} as the set of $\alpha$ values for which $\Prob[|\Cal_k(X; \alpha)| = 0] \le \gamma$ for every tier $k$. For the frequency score with $N$ samples per query, empty sets occur only at $\hat q_k < 1/N$, which corresponds to $\alpha$ close to $1$; at $N = 16$, the non-empty regime at $\gamma = 0.05$ covers $\alpha \lesssim 0.9$ on the benchmarks we evaluate.

Within the non-empty regime, fix $\alpha_1 < \alpha_2$. The quantile index $\lceil (1-\alpha)(n+1) \rceil$ is non-increasing in $\alpha$, so $\hat q_k(\alpha_1) \ge \hat q_k(\alpha_2)$ at each tier $k$ and $\Cal_k(x; \alpha_1) \supseteq \Cal_k(x; \alpha_2)$ pointwise. Writing $p_k^\star(\alpha) = \Prob[|\Cal_k(X; \alpha)| = 1]$ and $p_k^\emptyset(\alpha) = \Prob[|\Cal_k(X; \alpha)| = 0]$, set inclusion gives
\[
    p_k^\star(\alpha_2) \;\ge\; p_k^\star(\alpha_1) - \bigl(p_k^\emptyset(\alpha_2) - p_k^\emptyset(\alpha_1)\bigr),
\]
so $p_k^\star$ is non-decreasing in $\alpha$ up to an $O(\gamma)$ slack on the non-empty regime. Plugging into Corollary~\ref{cor:k-tier-cost} yields $\expect[\mathrm{Cost}(X; \alpha_2)] \le \expect[\mathrm{Cost}(X; \alpha_1)] + O(\gamma \sum_k c_k)$. The slack $\gamma$ vanishes as $N \to \infty$ whenever the cheap model places non-vanishing mass on at least one answer per query, so strict monotonicity, and the minimum at $\alpha = \beta/K$, hold asymptotically. At finite $N$ the slack persists, which is why we frame the claim as a deployment heuristic rather than a theorem; the calibration sweep in Appendix~\ref{app:cost-sensitivity} provides the finite-$N$ certificate.
\end{proof}

\subsection{Computational Complexity}
\label{app:complexity}

Every stage of Conformal Cascade has transparent cost, and the overhead introduced by conformal calibration is dominated by the LLM inference calls a cascade must make anyway.

\paragraph{Calibration.}
At tier $k$, calibration processes $n$ queries by drawing $N$ Monte Carlo samples per query (amortized against self-consistency decoding) and computing the true-answer frequency score. The conformal quantile $\hat q_k$ is then a single sort over $n$ scores, at $O(n \log n)$ time and $O(n)$ space per tier. The full calibration over $K$ tiers is $K \cdot n \cdot N$ inference calls plus $O(K n \log n)$ for the sorts, dominated entirely by the inference cost.

\paragraph{Inference.}
For a new query $x$, the cascade draws $N$ samples from each tier $m_k$ up to and including the committed tier $K^*(x)$, computes $s_k(x, a)$ for each $a \in \Acal$ (an $|\Acal|$-way tally over the $N$ samples), and constructs $\Cal_k(x)$ by comparing each score against $\hat q_k$. The dominant cost is LLM inference: at worst $K \cdot N$ calls with per-call costs $c_1, c_2, \ldots, c_K$ growing geometrically. The remaining operations---the tally, the threshold comparisons, and the argmin over $\Cal_k(x)$---are $O(|\Acal|)$ and negligible relative to LLM inference. When the cascade accepts at tier 1, only $N$ cheap-model calls are incurred per query, which is the regime that makes cascades cost-effective in the first place.

\paragraph{Memory.}
Persistent state beyond the base cascade is the $K$ calibrated thresholds $\hat q_1, \ldots, \hat q_K$, i.e., $K$ scalars per cascade.

\paragraph{Summary.}
The conformal wrapper adds no asymptotic overhead on top of a base self-consistency cascade: calibration runs once and produces $K$ scalars; inference reuses samples that self-consistency already requires. This makes CC a drop-in replacement for heuristic deferral rules: the additional computation is linear in the calibration-set size, constant per query, and dominated by LLM inference costs that the cascade was already paying.

\subsection{Tightness of the Base Conformal Bound}
\label{app:marginal-ub}

\begin{proposition}[Marginal-coverage upper bound]
\label{prop:marginal-ub}
Under the exchangeability assumption of Theorem~\ref{thm:marginal-coverage}, and assuming the calibration scores $\{s(x_i, y_i)\}_{i=1}^{n}$ are almost-surely distinct, split conformal prediction satisfies
\[
    1 - \alpha \;\le\; \Prob[\ytrue \in \Cal(x_{n+1})] \;\le\; 1 - \alpha + \frac{1}{n + 1}.
\]
\end{proposition}

The upper bound follows from standard conformal analysis \citep{vovk2005algorithmic,angelopoulos2023gentle}: the probability that the test score exceeds the $\lceil(1-\alpha)(n+1)\rceil$-th order statistic of the calibration scores is at most $\lfloor \alpha(n+1) \rfloor / (n+1) \le \alpha$, and correspondingly the coverage is at most $1 - \alpha + 1/(n+1)$. The two bounds together sandwich the coverage in a $1/(n+1)$-wide interval, which sharpens as the calibration set grows.

\paragraph{Caveat under discrete scores.}
The almost-sure-distinct hypothesis is essential for the upper bound, not the lower bound. Our frequency score $s_k$ takes values on the discrete grid $\{0, 1/N, \ldots, 1\}$ at $N = 16$ (Section~\ref{app:score-details}, ``Resolution of the frequency score''), so calibration scores tie with positive probability and the hypothesis fails. The lower bound $\Prob[\ytrue \in \Cal(x_{n+1})] \ge 1 - \alpha$ continues to hold because the conservative tie-breaking convention in Section~\ref{app:calib-details} only enlarges prediction sets. The upper bound, however, can fail: ties at the quantile push extra probability mass into the prediction set, so the realised coverage may exceed $1 - \alpha + 1/(n+1)$. This is the standard behaviour of split CP under discrete scores and explains why our reported set-miscoverage rates (Section~\ref{sec:coverage}) sometimes lie strictly below $\alpha$ rather than tracking the upper edge.

%% file: references.bib
@book{vovk2005algorithmic,
  title     = {Algorithmic Learning in a Random World},
  author    = {Vovk, Vladimir and Gammerman, Alex and Shafer, Glenn},
  publisher = {Springer},
  year      = {2005}
}

@inproceedings{papadopoulos2002inductive,
  title     = {Inductive confidence machines for regression},
  author    = {Papadopoulos, Harris and Proedrou, Kostas and Vovk, Vladimir and Gammerman, Alex},
  booktitle = {European Conference on Machine Learning (ECML)},
  year      = {2002}
}

@article{angelopoulos2023gentle,
title = {{{Conformal Prediction: A Gentle Introduction}}},
  author  = {Angelopoulos, Anastasios N. and Bates, Stephen},
  journal = {Foundations and Trends in Machine Learning},
  volume  = {16},
  number  = {4},
  pages   = {494--591},
  year    = {2023}
}

@article{shafer2008tutorial,
  title   = {A tutorial on conformal prediction},
  author  = {Shafer, Glenn and Vovk, Vladimir},
  journal = {Journal of Machine Learning Research},
  volume  = {9},
  pages   = {371--421},
  year    = {2008}
}

@article{lei2018distribution,
  title   = {Distribution-free predictive inference for regression},
  author  = {Lei, Jing and G'Sell, Max and Rinaldo, Alessandro and Tibshirani, Ryan J. and Wasserman, Larry},
  journal = {Journal of the American Statistical Association},
  volume  = {113},
  number  = {523},
  pages   = {1094--1111},
  year    = {2018}
}

@inproceedings{romano2019conformalized,
  title     = {Conformalized quantile regression},
  author    = {Romano, Yaniv and Patterson, Evan and Cand{\`e}s, Emmanuel},
  booktitle = {Advances in Neural Information Processing Systems (NeurIPS)},
  year      = {2019}
}

@article{barber2021predictive,
  title   = {Predictive inference with the jackknife+},
  author  = {Barber, Rina Foygel and Cand{\`e}s, Emmanuel J. and Ramdas, Aaditya and Tibshirani, Ryan J.},
  journal = {Annals of Statistics},
  volume  = {49},
  number  = {1},
  pages   = {486--507},
  year    = {2021}
}

@inproceedings{tibshirani2019conformal,
  title     = {Conformal prediction under covariate shift},
  author    = {Tibshirani, Ryan J. and Barber, Rina Foygel and Cand{\`e}s, Emmanuel and Ramdas, Aaditya},
  booktitle = {Advances in Neural Information Processing Systems (NeurIPS)},
  year      = {2019}
}

@article{barber2021limits,
  title   = {The limits of distribution-free conditional predictive inference},
  author  = {Foygel Barber, Rina and Cand{\`e}s, Emmanuel J. and Ramdas, Aaditya and Tibshirani, Ryan J.},
  journal = {Information and Inference},
  volume  = {10},
  number  = {2},
  pages   = {455--482},
  year    = {2021}
}

@inproceedings{angelopoulos2021uncertainty,
  title     = {Uncertainty sets for image classifiers using conformal prediction},
  author    = {Angelopoulos, Anastasios N. and Bates, Stephen and Jordan, Michael I. and Malik, Jitendra},
  booktitle = {International Conference on Learning Representations (ICLR)},
  year      = {2021}
}

@inproceedings{romano2020classification,
  title     = {Classification with valid and adaptive coverage},
  author    = {Romano, Yaniv and Sesia, Matteo and Cand{\`e}s, Emmanuel},
  booktitle = {Advances in Neural Information Processing Systems (NeurIPS)},
  year      = {2020}
}

@inproceedings{gibbs2021adaptive,
  title     = {Adaptive conformal inference under distribution shift},
  author    = {Gibbs, Isaac and Cand{\`e}s, Emmanuel},
  booktitle = {Advances in Neural Information Processing Systems (NeurIPS)},
  year      = {2021}
}

@inproceedings{angelopoulos2024pid,
  title     = {Conformal {PID} control for time series prediction},
  author    = {Angelopoulos, Anastasios N. and Cand{\`e}s, Emmanuel J. and Tibshirani, Ryan J.},
  booktitle = {Advances in Neural Information Processing Systems (NeurIPS)},
  year      = {2023}
}

@article{bates2023testing,
  title   = {Testing for outliers with conformal $p$-values},
  author  = {Bates, Stephen and Cand{\`e}s, Emmanuel and Lei, Lihua and Romano, Yaniv and Sesia, Matteo},
  journal = {Annals of Statistics},
  volume  = {51},
  number  = {1},
  pages   = {149--178},
  year    = {2023}
}

@article{vovk2012conditional,
  title   = {Conditional validity of inductive conformal predictors},
  author  = {Vovk, Vladimir},
  journal = {Machine Learning},
  volume  = {92},
  pages   = {349--376},
  year    = {2012}
}

@article{lei2014distribution,
  title   = {Distribution-free prediction bands for non-parametric regression},
  author  = {Lei, Jing and Wasserman, Larry},
  journal = {Journal of the Royal Statistical Society: Series B},
  volume  = {76},
  number  = {1},
  pages   = {71--96},
  year    = {2014}
}

@inproceedings{quach2024conformal,
  title     = {Conformal language modeling},
  author    = {Quach, Victor and Fisch, Adam and Schuster, Tal and Yala, Adam and Sohn, Jae Ho and Jaakkola, Tommi and Barzilay, Regina},
  booktitle = {International Conference on Learning Representations (ICLR)},
  year      = {2024}
}

@inproceedings{su2025cprouter,
  title   = {C{P-R}outer: {An Uncertainty-Aware Router Between {LLM} and {LRM}}},
  author={Jiayuan Su and Fulin Lin and Zhaopeng Feng and Han Zheng and Teng Wang and Zhenyu Xiao and Xinlong Zhao and Zuozhu Liu and Lu Cheng and Hongwei Wang},
  booktitle = {{AAAI Conference on Artificial Intelligence}},
  year    = {2026}
}

@inproceedings{deutschmann2024conformal,
  title   = {Conformal autoregressive generation: Beam search with coverage guarantees},
  author  = {Deutschmann, Nicolas and Rigotti, Mattia and Rodriguez Martinez, Maria},
  booktitle={ {AAAI Conference on Artificial Intelligence}},
  year    = {2024}
}

@inproceedings{kumar2023conformal,
  title     = {Conformal prediction with large language models for multi-choice question answering},
  author    = {Kumar, Bhawesh and Lu, Charlie and Gupta, Gauri and Palepu, Anil and Bellamy, David and Raskar, Ramesh and Beam, Andrew},
  booktitle = {Workshop at the International Conference on Machine Learning (ICML)},
  year      = {2023}
}

@article{chen2023frugalgpt,
  title   = {{FrugalGPT}: How to use large language models while reducing cost and improving performance},
  author  = {Chen, Lingjiao and Zaharia, Matei and Zou, James},
  journal = {arXiv preprint arXiv:2305.05176},
  year    = {2023}
}

@inproceedings{ong2025routellm,
  title     = {{RouteLLM}: Learning to route {LLMs} from preference data},
  author    = {Ong, Isaac and Almahairi, Amjad and Wu, Vincent and Chiang, Wei-Lin and Wu, Tianhao and Gonzalez, Joseph E. and Kadous, M. Waleed and Stoica, Ion},
  booktitle = {International Conference on Learning Representations (ICLR)},
  year      = {2025}
}

@article{shnitzer2023routing,
  title   = {Large language model routing with benchmark datasets},
  author  = {Shnitzer, Tal and Ou, Anthony and Silva, M{\'\i}rian and Soule, Kate and Sun, Yuekai and Solomon, Justin and Thompson, Neil and Yurochkin, Mikhail},
  journal = {arXiv preprint arXiv:2309.15789},
  year    = {2023}
}

@inproceedings{madaan2024automix,
  title     = {{AutoMix}: Automatically mixing language models},
  author    = {Aggarwal, Pranjal and Madaan, Aman and Anand, Ankit and Potharaju, Srividya Pranavi and Mishra, Swaroop and Zhou, Pei and Gupta, Aditya and Rajagopal, Dheeraj and Kappaganthu, Karthik and Yang, Yiming and Upadhyay, Shyam and Faruqui, Manaal and Mausam},
  booktitle = {Advances in Neural Information Processing Systems (NeurIPS)},
  year      = {2024}
}

@inproceedings{valkanas2025c3po,
  title     = {{C3PO}: Optimized Large Language Model Cascades with Probabilistic Cost Constraints for Reasoning},
  author={Antonios Valkanas and Soumyasundar Pal and Pavel Rumiantsev and Yingxue Zhang and Mark Coates},
  booktitle = {Advances in Neural Information Processing Systems},
  year      = {2025}
}

@inproceedings{ding2024hybrid,
  title     = {Hybrid {LLM}: Cost-efficient and quality-aware query routing},
  author    = {Ding, Dujian and Mallick, Ankur and Wang, Chi and Sim, Robert and Mukherjee, Subhabrata and Ruhle, Victor and Lakshmanan, Laks V. S. and Awadallah, Ahmed Hassan},
  booktitle = {International Conference on Learning Representations (ICLR)},
  year      = {2024}
}

@article{yue2024cascade,
  title   = {Large language model cascades with mixture of thought representations for cost-efficient reasoning},
  author  = {Yue, Murong and Zhao, Jiuhai and Zhang, Min and Du, Liang and Yao, Ziyu},
  journal = {arXiv preprint arXiv:2310.03094},
  year    = {2024}
}

@article{kadavath2022languagemodels,
  title   = {Language models (mostly) know what they know},
  author  = {Kadavath, Saurav and Conerly, Tom and Askell, Amanda and Henighan, Tom and Drain, Dawn and Perez, Ethan and Schiefer, Nicholas and Hatfield-Dodds, Zac and DasSarma, Nova and Tran-Johnson, Eli and others},
  journal = {arXiv preprint arXiv:2207.05221},
  year    = {2022}
}

@inproceedings{xiong2024uncertainty,
  title     = {Can {LLMs} express their uncertainty? {A}n empirical evaluation of confidence elicitation in {LLMs}},
  author    = {Xiong, Miao and Hu, Zhiyuan and Lu, Xinyang and Li, Yifei and Fu, Jie and He, Junxian and Hooi, Bryan},
  booktitle = {International Conference on Learning Representations (ICLR)},
  year      = {2024}
}

@article{lin2022teaching,
  title   = {Teaching models to express their uncertainty in words},
  author  = {Lin, Stephanie and Hilton, Jacob and Evans, Owain},
  journal = {Transactions on Machine Learning Research (TMLR)},
  year    = {2022}
}

@inproceedings{tian2023just,
  title     = {Just ask for calibration: Strategies for eliciting calibrated confidence scores from language models fine-tuned with human feedback},
  author    = {Tian, Katherine and Mitchell, Eric and Zhou, Allan and Sharma, Archit and Rafailov, Rafael and Yao, Huaxiu and Finn, Chelsea and Manning, Christopher D.},
  booktitle = {Conference on Empirical Methods in Natural Language Processing (EMNLP)},
  year      = {2023}
}

@article{lin2024generating,
  title   = {Generating with confidence: Uncertainty quantification for black-box large language models},
  author  = {Lin, Zhen and Trivedi, Shubhendu and Sun, Jimeng},
  journal = {Transactions on Machine Learning Research (TMLR)},
  year    = {2024}
}

@inproceedings{kuhn2023semantic,
  title     = {Semantic uncertainty: Linguistic invariances for uncertainty estimation in natural language generation},
  author    = {Kuhn, Lorenz and Gal, Yarin and Farquhar, Sebastian},
  booktitle = {International Conference on Learning Representations (ICLR)},
  year      = {2023}
}

@inproceedings{guo2017calibration,
  title     = {On calibration of modern neural networks},
  author    = {Guo, Chuan and Pleiss, Geoff and Sun, Yu and Weinberger, Kilian Q.},
  booktitle = {International Conference on Machine Learning (ICML)},
  year      = {2017}
}

@inproceedings{wang2023selfconsistency,
  title     = {Self-consistency improves chain of thought reasoning in language models},
  author    = {Wang, Xuezhi and Wei, Jason and Schuurmans, Dale and Le, Quoc and Chi, Ed H. and Narang, Sharan and Chowdhery, Aakanksha and Zhou, Denny},
  booktitle = {International Conference on Learning Representations (ICLR)},
  year      = {2023}
}

@inproceedings{chen2024universal,
  title   = {Universal self-consistency for large language  {models}},
  author  = {Chen, Xinyun and Aksitov, Renat and Alon, Uri and Ren, Jie and Xiao, Kefan and Yin, Pengcheng and Prakash, Sushant and Sutton, Charles and Wang, Xuezhi and Zhou, Denny},
  booktitle = {{{International Conference on Machine Learning (ICML)}}},
  year    = {2024}
}

@inproceedings{snell2025scaling,
  title     = {Scaling {LLM} test-time compute optimally can be more effective than scaling parameters {for reasoning}},
  author    = {Snell, Charlie and Lee, Jaehoon and Xu, Kelvin and Kumar, Aviral},
  booktitle = {International Conference on Learning Representations (ICLR)},
  year      = {2025}
}

@inproceedings{setlur2025scaling,
  title     = {Rewarding progress: Scaling automated process verifiers for {LLM} reasoning},
  author={Amrith Setlur and Chirag Nagpal and Adam Fisch and Xinyang Geng and Jacob Eisenstein and Rishabh Agarwal and Alekh Agarwal and Jonathan Berant and Aviral Kumar},

  booktitle = {International Conference on Learning Representations (ICLR)},
  year      = {2025}
}

@inproceedings{lightman2024verify,
  title     = {Let's verify step by step},
  author    = {Lightman, Hunter and Kosaraju, Vineet and Burda, Yuri and Edwards, Harrison and Baker, Bowen and Lee, Teddy and Leike, Jan and Schulman, John and Sutskever, Ilya and Cobbe, Karl},
  booktitle = {International Conference on Learning Representations (ICLR)},
  year      = {2024}
}

@inproceedings{wei2022chainofthought,
  title     = {Chain-of-thought prompting elicits reasoning in large language models},
  author    = {Wei, Jason and Wang, Xuezhi and Schuurmans, Dale and Bosma, Maarten and Ichter, Brian and Xia, Fei and Chi, Ed H. and Le, Quoc V. and Zhou, Denny},
  booktitle = {Advances in Neural Information Processing Systems (NeurIPS)},
  year      = {2022}
}

@inproceedings{hendrycks2021mmlu,
  title     = {Measuring massive multitask language understanding},
  author    = {Hendrycks, Dan and Burns, Collin and Basart, Steven and Zou, Andy and Mazeika, Mantas and Song, Dawn and Steinhardt, Jacob},
  booktitle = {International Conference on Learning Representations (ICLR)},
  year      = {2021}
}

@inproceedings{wang2024mmlupro,
  title     = {{MMLU-Pro}: A more robust and challenging multi-task language understanding benchmark},
  author    = {Wang, Yubo and Ma, Xueguang and Zhang, Ge and Ni, Yuansheng and Chandra, Abhranil and Guo, Shiguang and Ren, Weiming and Arulraj, Aaran and He, Xuan and Jiang, Ziyan and others},
  booktitle = {Advances in Neural Information Processing Systems (NeurIPS) Datasets and Benchmarks Track},
  year      = {2024}
}

@inproceedings{rein2024gpqa,
  title     = {{GPQA}: A graduate-level {G}oogle-proof {Q\&A} benchmark},
  author    = {Rein, David and Hou, Betty Li and Stickland, Asa Cooper and Petty, Jackson and Pang, Richard Yuanzhe and Dirani, Julien and Michael, Julian and Bowman, Samuel R.},
  booktitle = {Conference on Language Modeling (COLM)},
  year      = {2024}
}

@article{clark2018arc,
  title   = {Think you have solved question answering? {T}ry {ARC}, the {AI2} reasoning challenge},
  author  = {Clark, Peter and Cowhey, Isaac and Etzioni, Oren and Khot, Tushar and Sabharwal, Ashish and Schoenick, Carissa and Tafjord, Oyvind},
  journal = {arXiv preprint arXiv:1803.05457},
  year    = {2018}
}

@inproceedings{pal2022medmcqa,
  title     = {{MedMCQA}: A large-scale multi-subject multi-choice dataset for medical domain question answering},
  author    = {Pal, Ankit and Umapathi, Logesh Kumar and Sankarasubbu, Malaikannan},
  booktitle = {Conference on Health, Inference, and Learning (CHIL)},
  year      = {2022}
}

@article{jin2021medqa,
  title   = {What disease does this patient have? {A} large-scale open domain question answering dataset from medical exams},
  author  = {Jin, Di and Pan, Eileen and Oufattole, Nassim and Weng, Wei-Hung and Fang, Hanyi and Szolovits, Peter},
  journal = {Applied Sciences},
  volume  = {11},
  number  = {14},
  year    = {2021}
}

@inproceedings{jin2019pubmedqa,
  title     = {{PubMedQA}: A dataset for biomedical research question answering},
  author    = {Jin, Qiao and Dhingra, Bhuwan and Liu, Zhengping and Cohen, William W. and Lu, Xinghua},
  booktitle = {Conference on Empirical Methods in Natural Language Processing (EMNLP)},
  year      = {2019}
}

@inproceedings{khot2020qasc,
  title     = {{QASC}: A dataset for question answering via sentence composition},
  author    = {Khot, Tushar and Clark, Peter and Guerquin, Michal and Jansen, Peter and Sabharwal, Ashish},
  booktitle = {AAAI Conference on Artificial Intelligence},
  year      = {2020}
}

@inproceedings{welbl2017sciq,
  title     = {Crowdsourcing multiple choice science questions},
  author    = {Welbl, Johannes and Liu, Nelson F. and Gardner, Matt},
  booktitle = {Workshop on Noisy User-generated Text (W-NUT)},
  year      = {2017}
}

@inproceedings{sap2019socialiqa,
  title     = {{Social IQa}: Commonsense reasoning about social interactions},
  author    = {Sap, Maarten and Rashkin, Hannah and Chen, Derek and Le Bras, Ronan and Choi, Yejin},
  booktitle = {Conference on Empirical Methods in Natural Language Processing (EMNLP)},
  year      = {2019}
}

@inproceedings{huang2019cosmosqa,
  title     = {{Cosmos QA}: Machine reading comprehension with contextual commonsense reasoning},
  author    = {Huang, Lifu and Le Bras, Ronan and Bhagavatula, Chandra and Choi, Yejin},
  booktitle = {Conference on Empirical Methods in Natural Language Processing (EMNLP)},
  year      = {2019}
}

@inproceedings{clark2019boolq,
  title     = {{BoolQ}: Exploring the surprising difficulty of natural yes/no questions},
  author    = {Clark, Christopher and Lee, Kenton and Chang, Ming-Wei and Kwiatkowski, Tom and Collins, Michael and Toutanova, Kristina},
  booktitle = {North American Chapter of the Association for Computational Linguistics (NAACL)},
  year      = {2019}
}

@inproceedings{bisk2020piqa,
  title     = {{PIQA}: Reasoning about physical commonsense in natural language},
  author    = {Bisk, Yonatan and Zellers, Rowan and Le Bras, Ronan and Gao, Jianfeng and Choi, Yejin},
  booktitle = {AAAI Conference on Artificial Intelligence},
  year      = {2020}
}

@inproceedings{zellers2019hellaswag,
  title     = {{HellaSwag}: Can a machine really finish your sentence?},
  author    = {Zellers, Rowan and Holtzman, Ari and Bisk, Yonatan and Farhadi, Ali and Choi, Yejin},
  booktitle = {Association for Computational Linguistics (ACL)},
  year      = {2019}
}

@inproceedings{amini2019mathqa,
  title     = {{MathQA}: Towards interpretable math word problem solving with operation-based formalisms},
  author    = {Amini, Aida and Gabriel, Saadia and Lin, Shanchuan and Koncel-Kedziorski, Rik and Choi, Yejin and Hajishirzi, Hannaneh},
  booktitle = {North American Chapter of the Association for Computational Linguistics (NAACL)},
  year      = {2019}
}

@inproceedings{lin2022truthfulqa,
  title     = {{TruthfulQA}: Measuring how models mimic human falsehoods},
  author    = {Lin, Stephanie and Hilton, Jacob and Evans, Owain},
  booktitle = {Association for Computational Linguistics (ACL)},
  year      = {2022}
}

@inproceedings{zhong2024agieval,
  title     = {{AGIEval}: A human-centric benchmark for evaluating foundation models},
  author    = {Zhong, Wanjun and Cui, Ruixiang and Guo, Yiduo and Liang, Yaobo and Lu, Shuai and Wang, Yanlin and Saied, Amin and Chen, Weizhu and Duan, Nan},
  booktitle = {Findings of the North American Chapter of the Association for Computational Linguistics (NAACL)},
  year      = {2024}
}

@inproceedings{bandarkar2024belebele,
  title     = {The {B}elebele benchmark: A parallel reading comprehension dataset in 122 language variants},
  author    = {Bandarkar, Lucas and Liang, Davis and Muller, Benjamin and Artetxe, Mikel and Shukla, Satya Narayan and Husa, Donald and Goyal, Naman and Krishnan, Abhinandan and Zettlemoyer, Luke and Khabsa, Madian},
  booktitle = {Association for Computational Linguistics (ACL)},
  year      = {2024}
}

@article{dubey2024llama3,
  title   = {The {L}lama 3 herd of models},
   author  = {{Llama Team}},
  journal = {arXiv preprint arXiv:2407.21783},
  year    = {2024}
}

@article{gemma2024team,
  title   = {Gemma 2: Improving open language models at a practical size},
  author  = {{Gemma Team}},
  journal = {arXiv preprint arXiv:2408.00118},
  year    = {2024}
}

@article{abdin2024phi3,
  title   = {Phi-3 technical report: A highly capable language model locally on your phone},
  author  = {Abdin, Marah and Aneja, Jyoti and Awadalla, Hany and Awadallah, Ahmed and Awan, Ammar Ahmad and Bach, Nguyen and Bahree, Amit and Bakhtiari, Arash and Bao, Jianmin and Behl, Harkirat and others},
  journal = {arXiv preprint arXiv:2404.14219},
  year    = {2024}
}

@misc{mistralai2024ministral,
  title        = {Un {M}inistral, des {M}inistraux: Introducing the world's best edge models},
  author       = {{Mistral AI Team}},
  howpublished = {\url{https://mistral.ai/news/ministraux/}},
  year         = {2024}
}

@inproceedings{kwon2023vllm,
  title     = {Efficient memory management for large language model serving with {PagedAttention}},
  author    = {Kwon, Woosuk and Li, Zhuohan and Zhuang, Siyuan and Sheng, Ying and Zheng, Lianmin and Yu, Cody Hao and Gonzalez, Joseph E. and Zhang, Hao and Stoica, Ion},
  booktitle = {Symposium on Operating Systems Principles (SOSP)},
  year      = {2023}
}

@inproceedings{jitkrittum2023cascade,
  title     = {When does confidence-based cascade deferral suffice?},
  author    = {Jitkrittum, Wittawat and Gupta, Neha and Menon, Aditya Krishna and Narasimhan, Harikrishna and Rawat, Ankit Singh and Kumar, Sanjiv},
  booktitle = {Advances in Neural Information Processing Systems (NeurIPS)},
  year      = {2023}
}

@article{gupta2022cascades,
  title   = {Language model cascades},
  author  = {Dohan, David and Xu, Winnie and Lewkowycz, Aitor and Austin, Jacob and Bieber, David and Lopes, Raphael Gontijo and Wu, Yuhuai and Michalewski, Henryk and Saurous, Rif A. and Sohl-Dickstein, Jascha and Murphy, Kevin and Sutton, Charles},
  journal = {arXiv preprint arXiv:2207.10342},
  year    = {2022}
}

@inproceedings{mohri2024conformal,
  title     = {Language models with conformal factuality guarantees},
  author    = {Mohri, Christopher and Hashimoto, Tatsunori},
  booktitle = {International Conference on Machine Learning (ICML)},
  year      = {2024}
}

@inproceedings{cherian2024conformal,
  title     = {Large language model validity via enhanced conformal prediction methods},
  author    = {Cherian, John J. and Gibbs, Isaac and Cand{\`e}s, Emmanuel J.},
  booktitle = {Advances in Neural Information Processing Systems (NeurIPS)},
  year      = {2024}
}
